\documentclass[twoside]{article}
\usepackage[round, sort&compress]{natbib}
\usepackage[accepted]{aistats2026}
\usepackage{hyperref}
\usepackage{url}
\usepackage{booktabs}
\usepackage{amsfonts}
\usepackage{nicefrac}
\usepackage{xcolor}

\usepackage{graphicx}
\usepackage[justification=centering]{caption}
\usepackage{subcaption}
\usepackage{dsfont}
\usepackage{amsmath}
\usepackage{amssymb}
\usepackage{amsthm}
\usepackage{enumitem}
\usepackage{placeins}
\usepackage{verbatim}
\usepackage{multirow}
\usepackage{tikz}
\usepackage[ruled,vlined]{algorithm2e}
\usepackage{cleveref}
\usepackage{todonotes}

\newcommand{\ExampleBox}[1]{\todo[inline,inlinewidth=0.95\textwidth,color={rgb, 255:red, 223; green, 238; blue, 238 }]{#1}}

\newtheorem{theorem}{Theorem}[section]

\newtheorem{proposition}[theorem]{Proposition}

\theoremstyle{definition}

\newcommand{\PP}{\mathbb{P}}

\newcommand{\ud}{\, {\rm d} \kern-.015em }

\newcommand{\modulus}[1]{\left| \kern.05em #1 \kern.05em \right|}
\newcommand{\norm}[1]{\left\| \kern.05em #1 \kern.05em \right\|}
\newcommand{\inner}[1]{\left\langle \kern.05em #1 \kern.05em \right\rangle }

\newcommand{\pick}[2]{\renewcommand{\arraystretch}{0.6}
	\left( \kern-.4em \begin{array}{c} #1 \\ #2 \end{array} \kern-.4em \right) }

\begin{document}

%

%
\runningauthor{Trojan, Myshkov, Fearnhead, Hensman, Minka, Nemeth}

\twocolumn[

\aistatstitle{Scalable Model-Based Clustering with Sequential Monte Carlo}

\aistatsauthor{ Connie Trojan \And Pavel Myshkov \And  Paul Fearnhead }
\aistatsaddress{ Lancaster University \And  Microsoft Research \And  Lancaster University} 

\aistatsauthor{  James Hensman \And Tom Minka \And Christopher Nemeth }
\aistatsaddress{Microsoft Research \And Microsoft Research \And Lancaster University }
]

\begin{abstract}
In online clustering problems, there is often a large amount of uncertainty over possible cluster assignments that cannot be resolved until more data are observed. This difficulty is compounded when clusters follow complex distributions, as is the case with text data. Sequential Monte Carlo (SMC) methods give a natural way of representing and updating this uncertainty over time, but have prohibitive memory requirements for large-scale problems. We propose a novel SMC algorithm that decomposes clustering problems into approximately independent subproblems, allowing a more compact representation of the algorithm state. Our approach is motivated by the knowledge base construction problem, and we show that our method is able to accurately and efficiently solve clustering problems in this setting and others where traditional SMC struggles.
\end{abstract}

\section{INTRODUCTION}
\begin{figure*}
    \centering
    \includegraphics[height=4.5cm]{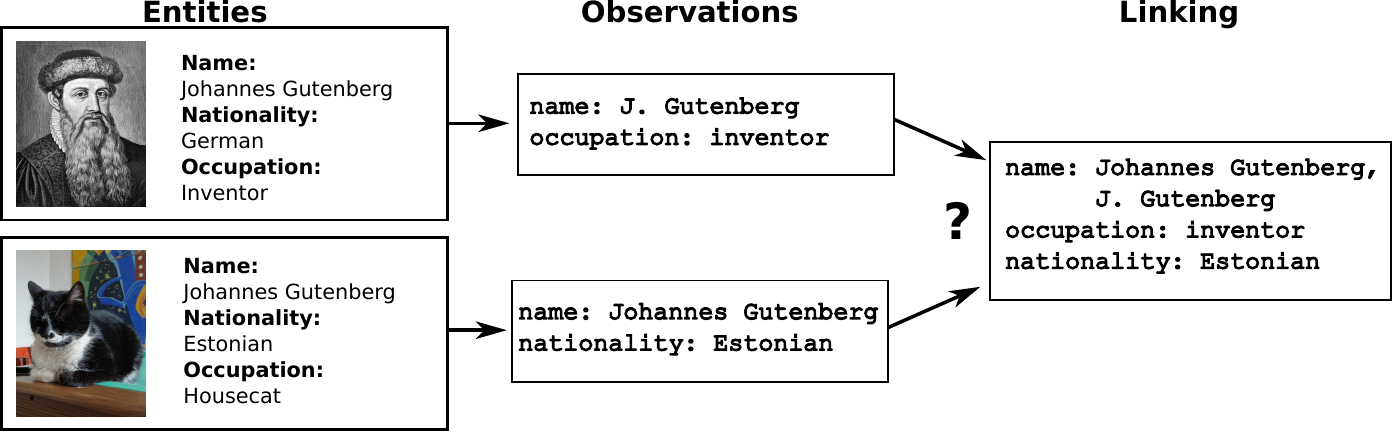}
    \caption{Entity disambiguation on Wikidata entries for “Johannes Gutenberg” \\ \citep[source: ][]{wikidata, gutenbergcat}}
    \label{fig:entity-linking}
\end{figure*}
\subsection{Knowledge Base Construction}
There has been much recent interest in augmenting large language models with external knowledge. This has two key advantages over storing knowledge implicitly in the model's parameters: it avoids the need to fine-tune the model through additional training, allowing easy updating of knowledge over time; and it improves the explainablility of the model's output by giving explicit sources for its information. This external knowledge typically takes the form of a corpus of documents, however recent work by \cite{wang2025kblam} has focused on providing the information in the form of a structured knowledge base. This approach has a number of advantages over methods that work directly with a text corpus -- it is more scalable than placing the entire corpus in the model's context window \citep[corpus-in-context prompting, ][]{lee2024longcontext}, and does not require running an external retriever to extract relevant documents from the corpus \citep[retrieval-augmented generation, ][]{lewis2020rag}. 

However, it requires the additional step of constructing a suitable knowledge base from an unstructured text corpus. This process involves extracting structured entity fragments, consisting of partial observations of entities' attributes, and grouping them together to form entries in the knowledge base \citep{weikum2010knowledge}. The latter task is known as entity linking or entity disambiguation. When an appropriate reference database exists, this is typically seen as a classification task aiming to predict to which existing entity new observations should be attributed \citep{sevgili2022entitylinking}.
However, many applications lack appropriate reference data, making the task of grouping entity fragments more challenging \citep{szekely2015kbctrafficking, winn2021enterprise}. Since each fragment contains partial, free text observations of an entity's attributes, there can be ambiguity about which observations refer to the same entity (see Figure \ref{fig:entity-linking}). 

In this setting, knowledge base construction can be seen as a clustering problem: partial observations must be grouped into clusters, each containing observations about a different entity. Even finding a maximum likelihood clustering in this setting is difficult, since the number of possible clusterings increases exponentially with dataset size. An additional challenge is that inference must often be performed in the online setting, since new entities may appear and entity properties can change over time \citep{weikum2010knowledge, zaporojets2022tempel}. Existing algorithms are typically based on hierarchical clustering, merging fragments together in a greedy fashion based on pairwise comparisons \citep{benjelloun2009swoosh}. However, this approach struggles in cases where there is significant uncertainty over possible clusterings that cannot be resolved until more information is observed, since merges cannot be undone without re-running the algorithm. As such, we focus on clustering algorithms designed for the online setting, which have the additional advantage of having reduced computational cost for large problems compared to offline methods.

\subsection{Model-based clustering}

Motivated by the knowledge base construction problem, we focus on solving clustering problems with a large, unknown underlying number of clusters, which require the use of a (potentially black-box) probabilistic model for cluster generation. We compare to both online and offline baselines in order to investigate the trade-off between accuracy, computational cost, and online capability.

Model-based clustering methods that do not require the number of clusters to be specified in advance are typically based on Bayesian nonparametric models, described in more detail in Section \ref{sec:Bayesian-nonparametrics}. Since the number of possible clusterings of the data (and hence the number of states in the posterior distribution) increases exponentially with dataset size, inference methods typically use sample-based approximations, or make simplifying assumptions in order to lower-bound the posterior density.

One commonly used approximate method is agglomerative hierarchical clustering \citep{jain1988clustering}, which considers clusterings that can be obtained by starting with every observation in a singleton cluster and iteratively merging until some truncation/stopping condition is met. 
Potential merges can be evaluated by computing hypothesis test statistics \citep{banfield1993clustering,heller2005bayesianclustering}, with a natural stopping criterion given by the point at which no more proposed merges are accepted by the hypothesis test. This gives a greedy algorithm for estimating maximum likelihood or maximum a posteriori clusterings.

Approximate Bayesian inference in this setting spans a range of techniques. Variational inference methods \citep{kenichi2007variationaldpmm,scherreik2020onlinebnp} optimise a simplified posterior approximation, yielding fast, scalable updates that readily extend to streaming data, but at the expense of restrictive modelling assumptions and potentially reduced accuracy to the true posterior. Inference relies on assuming independence between all latent variables and making alternating updates to assignments and parameters, so cannot be directly applied on black-box likelihoods without explicitly defined cluster parameters. 

By contrast, Monte Carlo algorithms sample cluster assignments exactly under the model, and do not impose restrictions on the form of the likelihood. Markov Chain Monte Carlo (MCMC) methods \citep{neal2000dpmmmcmc,bouchard2017splitmerge} iteratively sample individual assignments from their marginal distributions, but remain inherently offline and can converge slowly for large datasets. Distributed MCMC methods \citep{geetal2015distributeddpmm,Marchant03042021} can accelerate convergence, but typically rely on problem features that may not hold in the black-box setting (e.g. availability of summary statistics). Sequential Monte Carlo (SMC) offers an online sampling alternative which incrementally samples a particle approximation of the posterior \citep{Fearnhead2004mixture, caninietal2009LDAonline}. However, its memory and computational demands grow quickly with problem size \citep{kantas2015particle}.

\paragraph{Our contribution} We propose a novel online clustering framework that exploits the form of the posterior distribution to decompose large, complex datasets into independent subproblems and thereby enable efficient Sequential Monte Carlo inference. This problem decomposition is updated dynamically as new data arrive. Our method retains the asymptotic exactness of vanilla SMC, and we prove that the factorisation improves the reverse KL divergence to the full posterior, ensuring approximation fidelity is preserved. We demonstrate that our split SMC algorithm delivers substantial computational savings and improved clustering accuracy compared to standard approaches.

\section{BACKGROUND}
\subsection{Bayesian Clustering Models}\label{sec:Bayesian-nonparametrics}

Bayesian nonparametric clustering methods allow the number of clusters $k$ to grow with the data, making them well suited to tasks where the true cluster count is unknown.  A canonical example is the Dirichlet process mixture model (DPMM), see \cite{murphy2023pml2Book}.
Under a DPMM, we imagine data arriving sequentially: suppose that at time $t$, the existing data $x_{1:(t-1)}$ has cluster assignments $z_{1:(t-1)}$ partitioning it into clusters $ \{c_1, c_2, \dots, c_{K}\}$. A new observation $x_{t}$ appears in an existing cluster with probability
$\Pr(z_{t} = k \mid z_{1:(t-1)}) = \frac{|c_k|}{\alpha + t - 1},$
or a fresh cluster is created with probability
 $  \Pr(z_{t} = K+1 \mid z_{1:(t-1)}) = \frac{\alpha}{\alpha + t - 1},$
where $|c_k|$ denotes the size of cluster $k$ and $\alpha>0$ is a concentration parameter controlling the rate at which new clusters appear. Once its cluster membership is decided, the data point $x_t$ is drawn from the parametric distribution associated with its cluster, whose parameters themselves are drawn from a shared prior.

Inference under a DPMM typically involves evaluating, for a given data point $x_i$, the marginal probability of assignment to a given cluster conditional on the rest of the data and their clusters:
\begin{align}
    p(z_i = k \mid z_{-i}, x_{-i})
    &\propto \underbrace{\Pr(z_i = k \mid z_{-i})}_{\text{prior term}}
               \times
               \underbrace{p(x_i \mid x_{-i,k})}_{\text{likelihood term}}\,,
    \label{eq:marginaldist}
\end{align}
where $z_{-i}$ and $x_{-i}$ denote all assignments and observations except the $i$th, and $x_{-i,k}$ is the collection of these observations assigned to cluster $k$. The prior term follows directly from the above and the likelihood term requires integrating over any unknown cluster parameters. This likelihood can be obtained as the ratio $p(x_i, x_{-i,k})/p(x_{-i,k})$, where $p$ refers to the likelihood of observing a given unordered collection of datapoints in a sample of the specified size.  
This marginal form underpins both Gibbs sampling \citep{neal2000dpmmmcmc} and particle‐based sequential Monte Carlo (Section \ref{sec:smc}), since each update only requires local reassignment probabilities rather than enumeration of all partitions.

\subsection{Sequential Monte Carlo}\label{sec:smc}

Sequential Monte Carlo (SMC) provides an online particle‐based approximation to the evolving posterior distribution $p(z_{1:t}\mid x_{1:t})$, where the $z_t$ are unobserved latent variables (in our case, cluster assignments) and $x_t$ the corresponding observations \citep{Doucetetal2001SMCbook}. In clustering problems, since the latent variables are discrete, it is possible to compute exact updates to the posterior distribution at each timestep. \citet{Fearnhead2004mixture} exploit this to formulate an SMC clustering algorithm that guarantees each particle represents a unique clustering, thus avoiding redundant duplicate particles. 

In this algorithm, at each time $t$, we have access to a set of $m$ particles $\{p_{t-1}^{(i)}\}_{i=1}^m$ with weights $\{w_{t-1}^{(i)}\}_{i=1}^m$ that approximate the posterior at time $t-1$. Each particle contains a distinct clustering of $x_{1:(t-1)}$, which can be represented as either a vector $z_{1:(t-1)}$ of cluster assignments, or a partition of $x_{1:(t-1)}$ into clusters. When a new data point $x_t$ arrives, we perform three operations:

\begin{enumerate}
    \item \textbf{Propagation.} Construct a set of \textit{putative} particles, with one particle for each possible assignment of $x_t$ to a cluster on each existing particle.

\item \textbf{Weighting.}  For each putative particle $\tilde{p}_t(i,c_k)$, compute unnormalised putative weights using \eqref{eq:marginaldist}:
\begin{align*}
   \tilde w_t(i,c_k) 
    \propto 
   w_{t-1}^{(i)} \: p(z_t = k \mid p_{t-1}^{(i)}, x_{1:(t-1)}).    
\end{align*}

\item \textbf{Resampling.}  Draw $m$ new particles $\{p_{t}^{(i)}\}$ and their weights $\{w_t^{(i)}\}$ by sampling from the putative particle set according to the putative weights.
\end{enumerate}

This algorithm preserves the exact posterior up to the resampling step, which bounds both memory and computation by discarding low‐probability configurations. 
Many resampling schemes are possible in this framework \citep{Carpenter1999improved,Fearnhead-Clifford2003optimal}, trading off computational efficiency and variance reduction. Our algorithmic development (below) is based on making particle approximations that minimise reverse KL divergence to the true posterior. To be consistent with this, we use a simple greedy resampling scheme, where the $m$ particles with the highest weights are selected and their weights normalised -- this is the optimal way (in terms of reverse KL divergence) of resampling a particle approximation of fixed size $m$ from a discrete distribution (Appendix \ref{sec:resamplingproofs}). 

\section{SEQUENTIAL CLUSTERING METHODOLOGY}
\label{sec:method}

\paragraph{Overview} Applying standard SMC directly to large problems quickly becomes infeasible: the number of plausible clusterings, and hence the number of particles required for accurate inference, grows exponentially with dataset size, significantly increasing computational and memory demands. Repeated resampling erodes uncertainty over earlier observations so that only the most recent data remain well‐represented. Yet, under a Dirichlet process prior, well‐separated subsets of the data are close to independent: if they have vanishing probability of sharing a cluster, their joint posterior factorises. By partitioning the dataset into these approximately independent subproblems and maintaining separate particle approximations for each, we can implicitly represent a much larger particle set, dramatically cutting both memory and computational costs. In addition, updating and resampling only the subproblem affected by each new observation allows our algorithm to preserve uncertainty elsewhere.

Our algorithm proceeds similarly to that of \cite{Fearnhead2004mixture} described in Section \ref{sec:smc}, with the addition of a splitting step (Section \ref{sec:splitting}) at the end of each iteration to check whether the clustering problem can be decomposed. After such a split has occurred, we use a modified update step (Sections \ref{sec:update-step} and \ref{sec:merging}) that operates on the factorised representation. 

\begin{figure}
    \centering
    \includegraphics[height=3.25cm, width=3.25cm]{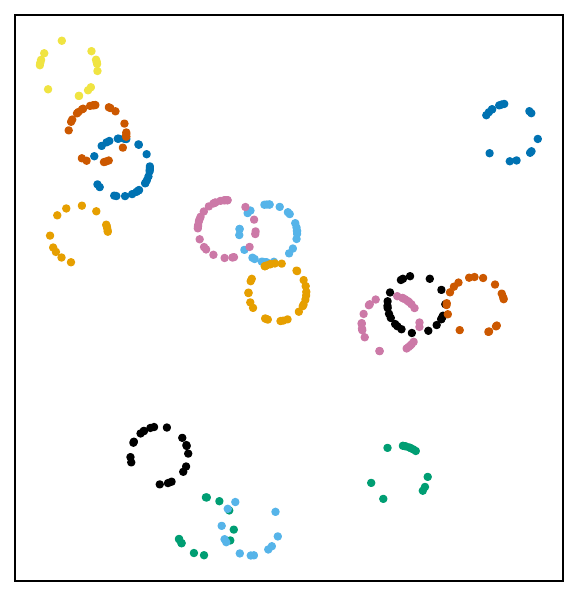}
    \includegraphics[height=3.25cm]{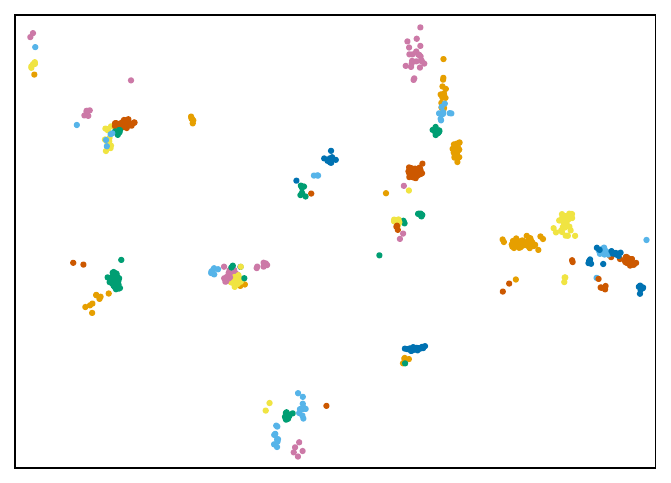}
    \\
    \includegraphics[height=3.25cm, width=3.25cm]{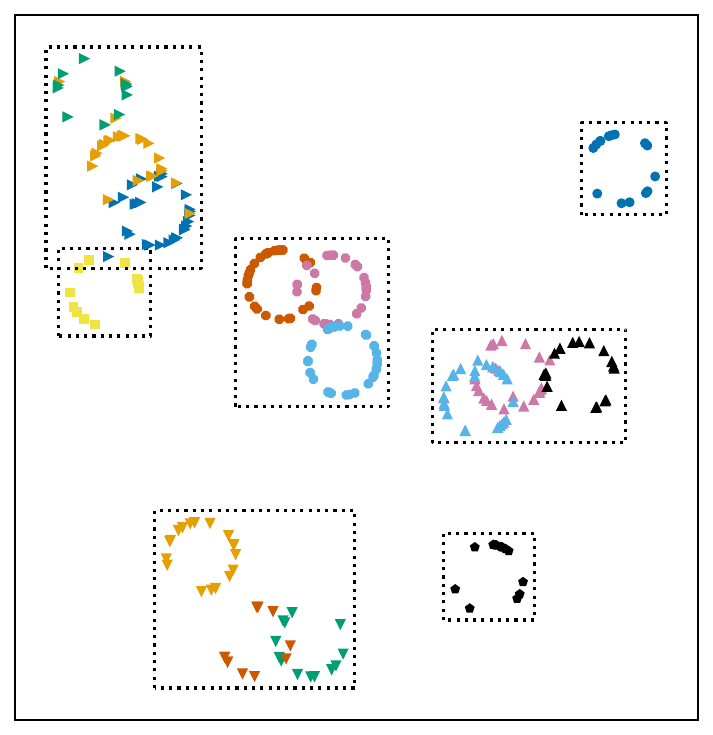}
    \includegraphics[height=3.25cm]{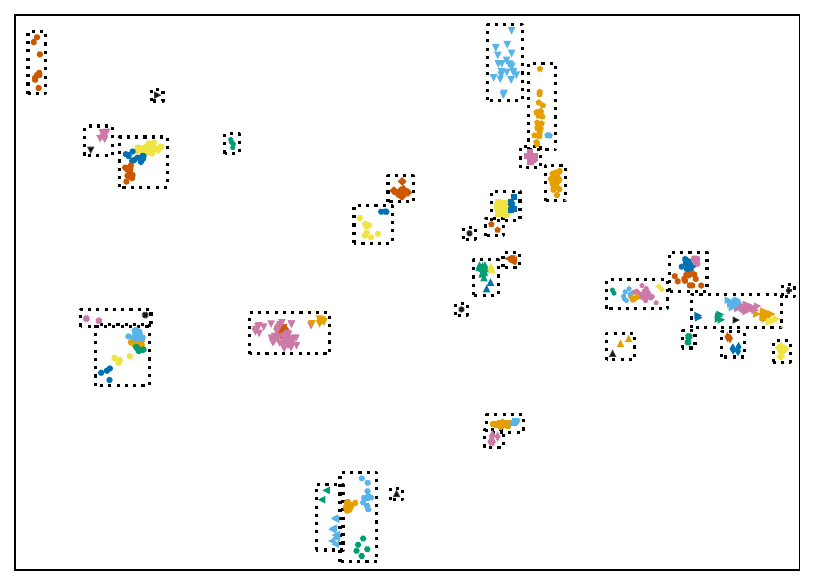}
    \caption{Top: scatterplots of the circles (left) and Gaussian mixture (right) datasets. Bottom: example split SMC clustering output, with subproblem partition
indicated in dotted lines}
    \label{fig:2d-data}
\end{figure}

\subsection{Notation}
\label{sec:notation}

We represent a clustering of the first $t$ observations $x_{1:t}$ by a partition
$p_t = \{c_1, c_2, \dots, c_{K}\}$, where each cluster $c_j\subseteq_{\text{m}} x_{1:t}$ and $\bigcup_{j=1}^K c_j = x_{1:t}$. We use the subscript m on set operations to clarify that we work with multisets of datapoints, since we consider datapoints with different time indices to be distinct from one another regardless of whether they have the same value. At time $t$, a particle approximation to the posterior has one such partition on each particle and consists of a weighted set
\begin{align*}
\mathcal{P}_t = \{\,p_t^{(i)}\}_{i=1}^m,\qquad
  \{\,w_t^{(i)}\}_{i=1}^m,
\end{align*}
where $p_t^{(i)}$ is the $i$th particle and $w_t^{(i)}$ its corresponding weight. 
In our splitting step, we further partition the data into $S$ disjoint subproblems $E_t = \{E_t^1, E_t^2, \dots, E_t^S\}$, with $\bigcup_{s=1}^S E_t^s = x_{1:t}$.  Each subproblem maintains its own particle set
\begin{align*}
  \mathcal{P}_t^s = \{\,p_t^{(s,i)}\}_{i=1}^{m_s},\qquad
  \{\,w_t^{(s,i)}\}_{i=1}^{m_s},
\end{align*}
where $p_t^{(s,i)}$ is a partition of only the observations in $E_t^s$, and $w_t^{(s,i)}$ is the corresponding weight. The maximum size of each subproblem particle set is fixed at a static value $m$, so the amount of information required to represent the stored clusterings is at most that of a vanilla particle filter with $m$ particles, which is linear in dataset size. The number of weights required, however, is linear in the number of subproblems rather than fixed.

\begin{figure*}
\centering
\includegraphics[height=4cm]{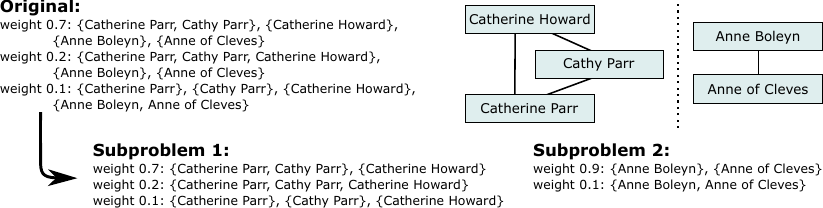}
\caption{Splitting into subproblems based on particle filter output}\label{fig:splitting}
\end{figure*}
\subsection{Split Step}
\label{sec:splitting}

Under a Dirichlet process prior, clusters are generated independently of each other. This means that if we condition on the event that two or more subsets of the data do not appear in clusters together, the true posterior distribution $p$ based on the Dirichlet process mixture model factorises into independent subproblems \citep{ewens1972sampling}. We can use this property to decompose the clustering problem: suppose that our approximate posterior $\hat{p}$ partitions the data into sub-problems and assigns no probability to data points from different sub-problems being in the same cluster. Then we can replace it with a distribution $\hat{p}^M$, that calculates the marginal distribution under $\hat{p}$ for clusters for each sub-problem, and defines the joint distribution over clusterings such that clusterings for sub-problems are independent. Since we know that the true posterior probability factorises in this way for each of the particles that make up $\hat{p}$, the factorised $\hat{p}^M$ has better accuracy.

\begin{proposition}[Informal statement]\label{prop:splitting}
Let $p$ denote the posterior distribution over clusterings of dataset $X$ induced by a Dirichlet process mixture model. Suppose we have a distribution $\hat{p}$ over clusterings of $X$, and a partitioning of $X$ into two or more disjoint subsets, such that $\hat{p}$ gives a probability of zero to any clustering where two datapoints in different partition elements appear in a cluster together. Let $\hat{p}^M$ be the distribution defined as the product of marginals of $\hat{p}$ for each partition element. Then $KL(\hat{p}^M||p)\leq KL(\hat{p}||p)$. Furthermore, $\hat{p}^M$ has the lowest KL divergence of any distribution with the same marginals as $\hat{p}$.
\end{proposition}
\begin{proof}
    With formal statement in Appendix \ref{sec:splittingproof}
\end{proof}

Splitting the clustering problem in this way allows us to implicitly represent a much larger particle set than could be represented explicitly with one particle per clustering of the full dataset, as the effective size of the new particle set is the product of the number of particles on each sub-problem (see Appendix \ref{sec:ess}). This involves no extra storage cost compared to the original SMC approximation, nor does it involve any additional computational cost at the next iteration to update the approximation with a new observation, but is guaranteed to give a better approximation.

The no-overlap event can be detected by constructing a graph whose vertices are observations and whose edges connect any two points that co‐occur in a cluster on at least one particle (see Figure~\ref{fig:splitting}).  We can efficiently obtain a partition of the data into subproblems $\{E_t^1,\dots,E_t^S\}$  by splitting this graph into its connected components \citep{pearce2005connectedcomponents}. Each component $E_t^s$ induces a set of clusters
\begin{align*}
  C_t^s
  \;&=\;
  \bigl\{\,c\in\bigcup_{p\in\mathcal P_t}p : c\subseteq E_t^s\bigr\},\\
\intertext{and hence a marginal particle set for its data,} 
  \mathcal P_t^s
  \;&=\;
  \{\,p\cap C_t^s \mid p\in\mathcal P_t\}, \\
  w_t^{(s,i)}
  &=\sum_{j=1}^{|\mathcal P_t|}w_t^{(j)}\,\mathds{1}[\,p_t^{(s,i)}\subseteq p_t^{(j)}\,]\,.
\end{align*}

Collectively, the $\mathcal P_t^s$ represent the full posterior approximation as the product space $\bigotimes_{s=1}^S\mathcal P_t^s$, with factorised weights.

Figure \ref{fig:2d-data} shows a visualisation of the resulting subproblem partitions on our 2D datasets. The subproblems are indicated by a bounding box drawn around their datapoints, as well as plotted with different markers. We highlight that the subproblem partition is decided dynamically by the posterior distribution rather than by a distance or similarity metric, since our algorithm aims to find the most accurate way of representing the joint clustering distribution with a limited amount of memory. As a result, the subproblems can be very close together in 2D space and may have nonlinear boundaries. Points that are nearest neighbours in a Euclidean sense are not necessarily placed in a subproblem together, since they may nevertheless have low posterior probability of being in the same cluster once the other datapoints are considered.

The splitting approach works best when the data have some degree of hierarchical structure: our algorithm makes splits when sub-populations of datapoints have greater within-population similarity than inter-population similarity, as the SMC sampler will discard inter-population clusterings before it discards within-population clusterings, and Proposition~\ref{prop:splitting} will apply when this happens.

\subsection{Update Step}
\label{sec:update-step}

In SMC clustering, when a new observation $x_t$ arrives, we first generate a \textit{putative} particle set by extending each previous particle $p_{t-1}^{(i)}\in\mathcal P_{t-1}$ in each possible way: (i) for each existing cluster $c\in p_{t-1}^{(i)}$ and (ii) for each new cluster $c=\varnothing$ (meaning $x_t$ forms a singleton). We define
\begin{align} \label{eq:weight}
      \tilde p_t(i,c)
  \;&=\;
  \{c\cup_{\textrm{m}} x_t\} \;\cup\; p^{(i)}_{t-1}\setminus\{c\},\nonumber
  \\
  \tilde w_t(i,c)
  \;&=\; w_{t-1}^{(i)} \;
  p\bigl(x_t\text{ belongs to } c \mid p_{t-1}^{(i)},x_{1:(t-1)}\bigr)\,,
\end{align}
where the weight update is given by \eqref{eq:marginaldist}.  Collecting the $\bigl\{\tilde p_t(i,c),\tilde w_t(i,c)\bigr\}$ for each cluster on each particle yields the putative particle set.

After a split has been made, we modify this procedure to work with the factorised particle representation. Having split the data into subproblems $E_{t-1}^1,\dots,E_{t-1}^S$, we can perform the same expansion independently on each subproblem $s$, forming a pooled set $\tilde{\mathcal{P}}_t$ of the putative particles $\tilde p_t{(s,i,c)}$ from each subproblem particle set, with weights $\tilde{w}_t{(s,i,c)}$ calculated as per (\ref{eq:weight}). The particles from each subproblem with the \textit{new cluster} assignment $c=\varnothing$ correspond to an identical global event, i.e. \textit{$x_t$ is a singleton cluster}. If we retained all such duplicates, we would overcount its weight. This can be avoided by only keeping the $\varnothing$ assignments from a single subproblem when forming $\tilde{\mathcal{P}}_t$, while discarding particles with $c=\varnothing$ from all other subproblems. We choose the subproblem with the highest weighted assignment for $x_t$.

After normalising the weights for the particle approximations $\mathcal{P}_{t-1}^s$ for each subproblem at time $t-1$, the particle approximation to the full posterior over clusterings of $x_{1:t}$ can be defined as follows. The probability of $x_t$ being added to subproblem $E_{t-1}^s$ and the clustering of this subproblem being $\tilde{p}_t{(s,i,c)}$ is proportional to $\tilde{w}_t{(s,i,c)}$; and then conditional on this, the clustering for subproblems $E_{t-1}^{s’}$, for $s’\neq s$, is drawn independently from the particle approximation $\mathcal{P}_{t-1}^{s’}$. 

There are two problems with this approximation. First, it places $x_t$ in all subproblems, violating the no-overlap condition that leads to independence across subproblems. Second, it involves more than $m$ particles for the clustering of $x_t$. To deal with the latter, we perform resampling, so as to maintain the required number of particles at each iteration. Then, based on whether the resampled particles obey the no-overlap condition, we may need to merge sub-problems.

A greedy resample of the putative particles simply involves choosing the $m$ values $\tilde{p}_t{(s,i,c)}$ that have largest weights $\tilde{w}_t{(s,i,c)}$. If the resampled particles all are within the same sub-problem, then our no-overlap condition holds with $x_t$ allocated to the appropriate subproblem, and we replace its particle set with the resampled set. If not, we perform a merge across subproblems that include $x_t$ as described in the next section. In either case, we then attempt the split step to further improve the partition into subproblems. We visualise this state update in Appendix \ref{sec:smc-details}.

\subsection{Merge Step}\label{sec:merging}

When the assignments for $x_t$ span multiple subproblems, the data partition is updated by merging the affected subproblems, which we index by $\mathcal{S} \subseteq \{1,\ldots,S\}$. This is done by producing a sample from their joint distribution, using the fact that this was independent for assignments prior to $x_t$. The previous split particle set $\{\mathcal{P}_{t-1}^s\}_{s\in\mathcal{S}}$ implicitly represents a larger particle set of solutions to the full clustering problem on $x_{1:(t-1)}$, formed as a Cartesian product of the subproblem particles. Hence, each updated particle $\tilde{p}_t{(s,i,c)}$ implicitly represents a joint distribution over clusterings of $x_{1:t}$, where the clustering of $x_t$ and the observations in $E^s_{t-1}$ has been fixed. To form an explicit representation of this joint distribution, we combine $\tilde{p}_t{(s,i,c)}$ with every possible particle configuration $p'$ drawn from the Cartesian product $\bigotimes_{j\in\mathcal{S}\setminus\{s\}}\mathcal{P}^j_{t-1}$. Then we have for each combination and weight
\begin{align*}
    p_{\rm joint} &= \tilde{p}_t{(s,i,c)} \;\cup\; p', \\
    w_{\rm joint} &= \tilde{w}_t{(s,i,c)} \;\times\; \prod_{j\in\mathcal{S}\setminus\{s\}} w_{t-1}^{(j,\cdot)}\,.
\end{align*}
Computing this expanded set for each $\tilde{p}_t{(s,i,c)}$ gives us an explicit form for the joint distribution represented by the output of the update step, across the subproblems in $\mathcal{S}$. Resampling this set yields a valid sample from the joint posterior over the affected subproblems. 

\paragraph{Checking merges} In this merge step, it is possible for a subproblem's assignments for $x_t$ to be discarded entirely -- in this case, the merge is undone by the next splitting step, with the existing particles on that subproblem thinned by the resample. We can improve the quality of the particle approximation by instead discarding particles where $x_t$ is assigned to a given subproblem if they have sufficiently low weight (see Appendix \ref{sec:resamplingproofs}). A rule of thumb is to set a threshold of $1/m$ on their combined weight, since in this case the expected number of times assignments on that subproblem will be resampled in the merging step is less than one, and the merge is likely to be redundant.

\paragraph{Multinomial merging} When $\lvert\mathcal{S}\rvert$ or the subproblem particle counts are large, performing merges by explicitly forming the full joint distribution can be expensive. A cheaper but higher-variance alternative is to use multinomial resampling in three stages: (i) sample $m$ assignments of $x_t$ (with replacement) from the marginal distribution over assignments, then (ii) for each sampled assignment, independently draw one particle from each of the other involved subproblems, and finally (iii) merge any duplicate particles. 
\pagebreak

\subsection{Surrogate Models}
\label{sec:implementation}

In domains such as entity linking, accurate cluster likelihoods often depend on complex generative models whose evaluation cost grows with cluster size, which can become prohibitive when every potential assignment must be scored in an SMC update. To mitigate this, we introduce lightweight surrogate models that can quickly identify the most plausible cluster assignments. By using the resulting DPMM posterior under this surrogate likelihood as a proposal distribution within each update step (details in Appendix \ref{sec:smc-details}), we concentrate neural likelihood evaluations on a smaller subset of high-probability assignments, drastically reducing model calls. We found that using a surrogate model in this way can even improve accuracy in the online setting (see Appendix \ref{sec:surrogate-ablation}).

For continuous data with a meaningful distance metric, Gaussian mixture models can be used. For text data, we use character‐level n-gram models with Dirichlet priors \citep{chien2015bayesngram} to capture overlap patterns in entity names. Surrogate models can also be applied as proposal distributions in vanilla SMC and MCMC sampling.

\section{EXPERIMENTS}
\begin{figure*}
    \centering{
    \includegraphics[height=4cm]{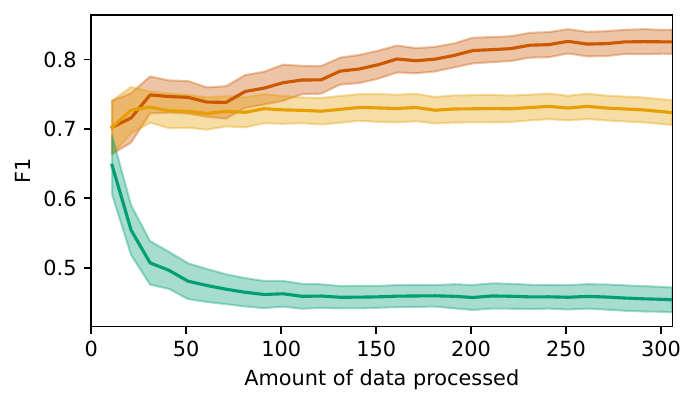}
\includegraphics[height=4cm]{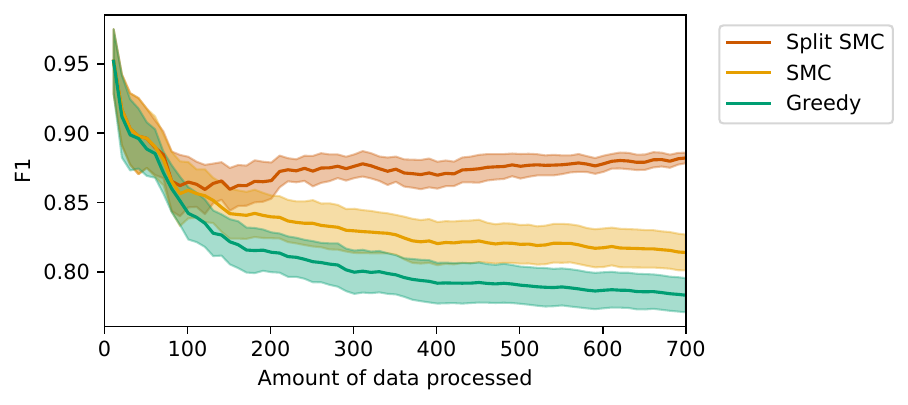}

    }
    \caption{Online accuracy for the circles (left) and GMM (right) datasets. We show the average F1 score of the highest-weighted clustering as more data is observed, with shaded regions indicating 95\% confidence intervals.}
    \label{fig:online-f1}
\end{figure*}

We compare our method to a range of standard approaches on two 2D clustering problems and three text datasets. Our baselines are as follows:
\begin{itemize}
\item \textbf{Sequential Monte Carlo} \citep{Fearnhead2004mixture} - the online clustering algorithm described in Section \ref{sec:smc}, equivalent to our proposed algorithm without the split step.
\item \textbf{Greedy clustering} - an online algorithm assigning each new datapoint to the cluster with the highest weight. This is a special case of SMC with one particle.
\item \textbf{Gibbs sampler} \citep{neal2000dpmmmcmc} - an offline MCMC algorithm that iteratively resamples individual cluster assignments using Equation \ref{eq:marginaldist}. Where a neural model is used, we also compare to a Metropolis-within-Gibbs version of this sampler, in which a surrogate model is used as a proposal distribution.
\item \textbf{Agglomerative clustering} \citep{banfield1993clustering} - an offline algorithm that iteratively merges the pair of clusters with the highest value of the Bayesian hypothesis test statistic comparing the merged and un-merged clusterings under the DPMM. This terminates when no merges are found that pass the hypothesis test.
\end{itemize}

Evaluation is based on model log-posterior density, runtime, and clustering accuracy based on precision, recall and F1 scores \citep{bagga1998entity}. The un-normalised log-posterior density is computed using Ewen's sampling formula (Equation \ref{eq:dpmm-post} in Appendix \ref{sec:splittingproof}). 

In MCMC inference, the sample with the highest posterior probability is used for this comparison. For the sequential Monte Carlo samplers, we use the highest-weighted particle in the particle set. The data were re-shuffled for each replication in order to compare performance over different orderings. Offline methods were run until termination where possible, or for a maximum runtime of $10^4$ seconds. We stopped the MCMC samplers when 500 sweeps passed without change to the estimated MAP clustering. 

We present a summary of the clustering algorithms' runtime and performance on log-posterior density and F1 in Table \ref{tab:metrics}. SMC results are reported for a fixed $m=100$ particles. An extended comparison of the full results including precision and recall is presented in Appendix \ref{sec:additional-results}, where we compare results for a wider range of hyperparameter configurations, including particle set size, to establish that our results are robust to changes in tuning and runtime budget. We give full experimental details in Appendix \ref{sec:experiment-details}, and full source code for our experiments is available on  \href{https://github.com/microsoft/smc-clustering}{GitHub}.

\subsection{Circles}\label{sec:circles}

In this experiment, clusters of observations are generated as points on a circle of fixed radius, with each cluster having a different location. This problem is particularly challenging when clusters overlap with each other, since neighbouring points can belong to different clusters and more than two points must be considered jointly to determine the number and location of clusters. A dataset of 300 points across 15 clusters was generated, giving a problem where several groups of clusters had non-trivial overlap with each other, see Figure \ref{fig:2d-data}. A variational diffusion model \citep{kingma2021variationaldiffusion} was used to approximate cluster likelihoods as the model likelihood of observing a collection of datapoints of a given size, with a Gaussian mixture model used as a surrogate.

Results across 30 replications are reported in Table \ref{tab:metrics}. The SMC samplers had equivalent or better performance on log-posterior density, recall, precision and F1 than the agglomerative algorithm. The MCMC sampler outperformed the sequential algorithms on log-posterior density and precision, at the expense of a runtime that was 10 times longer.
Split SMC significantly outperformed vanilla SMC on all metrics, and had higher recall and F1 than the MCMC sampler.

\subsection{Gaussian Mixture}

A dataset of 700 datapoints was generated from a Gaussian mixture model with a known distribution over cluster means and variances, plotted in Figure \ref{fig:2d-data}. Cluster assignments were generated from a Dirichlet process, resulting in 80 clusters with a skewed distribution of sizes.

Results across 20 replications are reported in Table \ref{tab:metrics}. The split SMC algorithm had equivalent or better performance on log-posterior density, recall, precision, and F1 compared to the offline algorithms. Split SMC significantly outperformed vanilla SMC, and the variance of the split SMC output was lower. Despite initially having very similar maximum a posteriori clusterings, the samplers' performance diverged significantly as more data was observed (Figure \ref{fig:online-f1}). In contrast to the other sequential algorithms, the split SMC sampler's output had increasing accuracy and decreasing variance, showing a lower dependence on the order in which the data were observed. We highlight that the split SMC sampler's performance was equivalent to that of offline agglomerative clustering, despite the fact that agglomerative algorithms are particularly well suited to the Gaussian setting.

\begin{table*}[]
\caption{Summary of accuracy metrics for each clustering algorithm across the datasets. The mean value of each metric across replications is reported, with its standard deviation in brackets.}
\label{tab:metrics}
\scriptsize{
\begin{center}
\begin{tabular}{lccccccccc}
\toprule
 & \multicolumn{3}{c}{Circles} & \multicolumn{3}{c}{Gaussian mixture} & \multicolumn{3}{c}{REBEL-50}\\ \cmidrule(lr){2-4}\cmidrule(lr){5-7}\cmidrule(lr){8-10}
Method & \begin{tabular}[c]{@{}c@{}}Log\\ posterior\end{tabular} & F1 & \begin{tabular}[c]{@{}c@{}}Runtime\\ (minutes)\end{tabular} & \begin{tabular}[c]{@{}c@{}}Log\\ posterior\end{tabular} & F1 & \begin{tabular}[c]{@{}c@{}}Runtime\\ (minutes)\end{tabular}  & \begin{tabular}[c]{@{}c@{}}Log\\ posterior\end{tabular} & F1 & \begin{tabular}[c]{@{}c@{}}Runtime\\ (minutes)\end{tabular} \\ \midrule
Greedy & -906 (175) & .46 (.06) & 1.4 & -1539 (34) & .78 (.03) & 0.2 & -8674 (27) & .66 (.01) & 1.4\\
SMC & -202 (92) & .73 (.04) & 1.6 & -1497 (30) & .81 (.03) & 0.4 & -8791 (52) & .65 (.01) & 1.4\\
Split SMC  & 10 (74) & \textbf{.82} (.05) & 5.1 & \textbf{-1424} (7) & \textbf{.88} (.01) & 1.7 & -8644 (17) & \textbf{.68} (.01) & 1.6\\
MCMC & \textbf{187} (15) & .81 (.02) & 54.7 & -1428 (2) & \textbf{.88} (.01) & 54.4 & -8709 (19) & .63 (.01) & 172.8\\
Agglomerative & -81 (48) & .73 (.04) & 54.7 & -1425 (0) & .87 (.00) & 239.9 & \textbf{-8543} (0) & .63 (.00)  & 172.2\\ \bottomrule
\end{tabular}
\\[\baselineskip]
\begin{tabular}{lcccccc}
\toprule
  & \multicolumn{3}{c}{REBEL-200}
 & \multicolumn{3}{c}{TweetNERD}\\ \cmidrule(lr){2-4}\cmidrule(lr){5-7}
Method & \begin{tabular}[c]{@{}c@{}}Log\\ posterior\end{tabular} & F1 & \begin{tabular}[c]{@{}c@{}}Runtime\\ (minutes)\end{tabular} & \begin{tabular}[c]{@{}c@{}}Log\\ posterior\end{tabular} & F1 & \begin{tabular}[c]{@{}c@{}}Runtime\\ (minutes)\end{tabular} \\ \midrule
Greedy  & -35382 (115) & .69 (.01) & 6.6 & -31780 (123) & .51 (.01)  & 163.7\\
SMC   & -36096 (137) & .69 (.01)  & 18.5 & -32923 (221) & \textbf{.56 (.01)} & 51.5\\
Split SMC  & \textbf{-35297} (108) & \textbf{.70} (.01) & 11.0 & \textbf{-31705} (110)& .55 (.01)&  128.0 \\
MCMC   & -- & -- & --  & -- & -- & -- \\
Agglomerative  & -- & -- & --  & -- & -- & -- \\  \bottomrule
\end{tabular}
\end{center}
}
\end{table*}

\subsection{Knowledge Base Construction}

We applied our algorithm to clustering datasets of entity fragments derived from the REBEL \citep{huguet-cabot-navigli-2021-rebel} and TweetNERD \citep{mishra2022tweetnerd} datasets, using a neural language model with a bigram surrogate (details in Appendix \ref{sec:kb-details}). We used two datasets based on REBEL: REBEL-50 and REBEL-200, which had 50 and 200 clusters, and 428 and 1671 observations respectively. They contained observations of entity fragments with a range of attributes, collected from Wikipedia. The TweetNERD data had 200 clusters and 1931 fragments containing mentions of entity names extracted from tweets. 

Both agglomerative clustering and MCMC failed to produce meaningful results on the larger datasets within our runtime budget of $10^4$ seconds, so are omitted from these comparisons. We include a full comparison on the smaller REBEL-50 dataset, where they could be run to convergence. 

Results are reported in Table \ref{tab:metrics}, across 10 replications for REBEL-50 and 20 for the larger datasets. A GPU was used in both REBEL experiments, so the runtimes for REBEL are much shorter. The language model was trained on a training split of the original REBEL dataset and achieved much lower recall clustering TweetNERD, likely due to the greater variation in entity names. On REBEL-50, split SMC was outperformed only by the agglomerative algorithm on log-posterior density, but had the highest F1 of any method. On the larger datasets, split SMC outperformed the other online algorithms on log-posterior density, though we note that the gap between it and greedy clustering on this metric was not statistically significant at the 95\% level on TweetNERD. Split SMC had the best F1 on REBEL-200, but was outperformed by vanilla SMC on TweetNERD.

\section{DISCUSSION}

Motivated by the knowledge base construction problem, we presented a novel sequential Monte Carlo algorithm for clustering and demonstrated its ability to scale to problems with complex cluster-likelihoods and a large number of clusters. Our split SMC algorithm detects and exploits near-independence in its approximation to the posterior distribution in order to produce a compact factorised representation of the algorithm state, and is therefore able to dynamically adjust the effective size of its own particle set without incurring prohibitive memory or computational cost.

We showed that for a fixed number of particles, split SMC is able to scale to clustering problems on which vanilla SMC struggles, achieving superior results on log-posterior density and, with the exception of the TweetNERD dataset, accuracy. Due to the larger number of unique clusterings represented by our algorithm and hence the larger number of model evaluations required, its runtime for a given particle set size is typically higher than that of vanilla SMC. We provide an extended analysis in Appendix \ref{sec:additional-results} showing that our algorithm also outperforms vanilla SMC for a wide range of runtimes and particle set sizes.

Since most of our experiments required working with neural likelihoods, we used simple surrogate models to reduce the number of likelihood evaluations required. This allowed accurate clusterings to be found much more efficiently, and we found that using surrogate models in this way actually improved accuracy in the online setting (see Appendix \ref{sec:surrogate-ablation}).

In addition, we found that split SMC was capable of achieving comparable results to offline methods despite operating in the online setting and requiring a substantially shorter runtime. In our experiments, it typically achieved equivalent or even superior results to the offline baselines, which in our larger experiments failed to converge or produce meaningful results within a reasonable runtime.

\section{Acknowledgements}
CT acknowledges the support of the EPSRC-funded EP/S022252/1 Centre for Doctoral Training in Statistics and Operational Research in Partnership with Industry (STOR-i), and Microsoft Research. PF was supported by EPSRC grants EP/Y028783/1, EP/R034710/1 and EP/R018561/1. CN was supported by EPSRC grants EP/Y028783/1 and EP/V022636/1 and the Research England Expanding Excellence in England (E3) funding for MARS: Mathematics for AI in Real-world Systems.

\FloatBarrier

{
\bibliographystyle{abbrvnat}
\bibliography{bibliography}

}
\section*{Checklist}



\begin{enumerate}

  \item For all models and algorithms presented, check if you include:
  \begin{enumerate}
    \item A clear description of the mathematical setting, assumptions, algorithm, and/or model. [\textcircled{Yes}/No/Not Applicable]\\

    Mathematical setting given in Section \ref{sec:Bayesian-nonparametrics}. Algorithm details given in Section \ref{sec:method} and Appendix \ref{sec:algorithm-details}. Assumptions for theoretical results given in Appendix \ref{sec:proofs}. Model details given in Appendix \ref{sec:data-details}.\\
    
    \item An analysis of the properties and complexity (time, space, sample size) of any algorithm. [\textcircled{Yes}/No/Not Applicable]\\

    Memory and runtime costs are discussed in Section \ref{sec:notation} and Appendix \ref{sec:runtime-metrics}.\\
    
    \item (Optional) Anonymized source code, with specification of all dependencies, including external libraries. [\textcircled{Yes}/No/Not Applicable]\\

    A package implementing our algorithm is available on GitHub at \href{https://github.com/microsoft/smc-clustering}{https://github.com/microsoft/smc-clustering}.\\ 
  \end{enumerate}

  \item For any theoretical claim, check if you include:
  \begin{enumerate}
    \item Statements of the full set of assumptions of all theoretical results. [\textcircled{Yes}/No/Not Applicable]
    \item Complete proofs of all theoretical results. [\textcircled{Yes}/No/Not Applicable]
    \item Clear explanations of any assumptions. [\textcircled{Yes}/No/Not Applicable]  \\

    Proofs and assumptions can be found in Appendix \ref{sec:proofs}.\\
  \end{enumerate}

  \item For all figures and tables that present empirical results, check if you include:
  \begin{enumerate}
    \item The code, data, and instructions needed to reproduce the main experimental results (either in the supplemental material or as a URL). [\textcircled{Yes}/No/Not Applicable]\\
    
    The experiment source code is available on GitHub. We give experimental details in Appendix \ref{sec:experiment-details}, describing the data used in Appendix \ref{sec:data-details}.\\ 
    
    \item All the training details (e.g., data splits, hyperparameters, how they were chosen). [\textcircled{Yes}/No/Not Applicable]\\

    Details in Appendix \ref{sec:data-details}.\\
    
    \item A clear definition of the specific measure or statistics and error bars (e.g., with respect to the random seed after running experiments multiple times). [\textcircled{Yes}/No/Not Applicable]\\

    These are given in the table and figure captions.\\
    
    \item A description of the computing infrastructure used. (e.g., type of GPUs, internal cluster, or cloud provider). [\textcircled{Yes}/No/Not Applicable]\\

    This can be found in Appendix \ref{sec:experiment-details}.\\
  \end{enumerate}

  \item If you are using existing assets (e.g., code, data, models) or curating/releasing new assets, check if you include:
  \begin{enumerate}
    \item Citations of the creator If your work uses existing assets. [\textcircled{Yes}/No/Not Applicable]
    \item The license information of the assets, if applicable. [\textcircled{Yes}/No/Not Applicable]\\

    See Appendix \ref{sec:kb-details}.\\
    
    \item New assets either in the supplemental material or as a URL, if applicable. [Yes/No/\textcircled{Not Applicable}]
    \item Information about consent from data providers/curators. [Yes/No/\textcircled{Not Applicable}]
    \item Discussion of sensible content if applicable, e.g., personally identifiable information or offensive content. [\textcircled{Yes}/No/Not Applicable]\\

    REBEL \citep{huguet-cabot-navigli-2021-rebel} and TweetNERD \citep{mishra2022tweetnerd} potentially contain personally identifiable information. Their contents are entirely drawn from text publicly available on Wikipedia and Twitter respectively.
  \end{enumerate}

  \item If you used crowdsourcing or conducted research with human subjects, check if you include:
  \begin{enumerate}
    \item The full text of instructions given to participants and screenshots. [Yes/No/\textcircled{Not Applicable}]
    \item Descriptions of potential participant risks, with links to Institutional Review Board (IRB) approvals if applicable. [Yes/No/\textcircled{Not Applicable}]
    \item The estimated hourly wage paid to participants and the total amount spent on participant compensation. [Yes/No/\textcircled{Not Applicable}]
  \end{enumerate}

\end{enumerate}

\clearpage
\appendix
\thispagestyle{empty}

\onecolumn
\aistatstitle{Scalable Model-Based Clustering with Sequential Monte Carlo: \\
Supplementary Materials}

\FloatBarrier
\appendix

\section{ADDITIONAL RESULTS}\label{sec:additional-results}

\subsection{Full Results and Runtime Comparison}\label{sec:runtime-metrics}

Here, we present the full results of each clustering algorithm on our experiments. Figures \labelcref{fig:circles-runtime,fig:gmm-runtime,fig:rebel50-runtime,fig:rebel200-runtime,fig:tweet-runtime} plot each metric against the runtime of each algorithm. 
For the offline methods, we plot the metrics every 10 iterations. We also compare to faster, approximate agglomerative algorithms that use subsampling (details in Section \ref{sec:algorithm-details}) and note their batch size in parentheses. The sequential methods were run for different particle set sizes (or model evaluation budgets in the case of greedy clustering) to measure performance across a range of possible runtimes. 

The star markers in the figures indicate the algorithm configurations reported in the main paper -- offline methods were run as close to convergence as possible (see Section \ref{sec:algorithm-details}), while SMC methods used a fixed particle set size of $m=100$. On the circles data, the Metropolis-within-Gibbs MCMC sampler achieved slightly superior results to exact Gibbs in a shorter runtime, so this is the result reported in the main paper for MCMC.

Section \ref{sec:particles} additionally presents the SMC results against the number of particles used. Our split SMC algorithm typically has a higher runtime than vanilla SMC for a fixed number of particles, primarily because it requires more model evaluations per iteration, since it stores a larger number of unique clusterings with the same number of particles. The splitting step also introduces a small additional computational overhead. We reported results for a fixed $m=100$ in the main paper for the sake of consistency across experiments, as the runtime difference was heavily dependent on the computational resources available in each experiment.  In addition, fixing $m$ highlighted that our algorithm is not typically particularly sensitive to the choice of this hyperparameter, since the effective number of particles is adjusted dynamically depending on the problem's characteristics. 

While algorithm runtimes are highly implementation-dependent, we present them to give an idea of how the various algorithm computational costs interact in practical settings -- for example: neural model evaluations, surrogate model evaluations, parallelism, and computational overhead from state updates. Agglomerative methods tend to require many more model evaluations than other algorithms, while Metropolis-within-Gibbs requires far fewer (only 1-2 per update, to compute an acceptance probability). However, the model evaluations in agglomerative and SMC clustering can be run in parallel batches at each iteration, while Gibbs updates are inherently sequential, leading to much longer runtimes for the same number of model evaluations. This difference can be seen in the REBEL experiments, which were run on a GPU - the runtime difference between SMC and split SMC is very small despite split SMC using more model evaluations, because the additional evaluations can be run in parallel.

In all experiments, it is possible to achieve reasonable accuracy with sequential methods in a much shorter time than with any offline method considered. In addition, there is a band of runtimes for which split SMC has equivalent or better results compared to every other method, on every metric except precision. Since the offline methods are initialised with every observation in a singleton cluster (giving a precision of 1), they have high precision at short runtimes. Vanilla SMC also has higher precision than split SMC on the text data, but lower recall.

\begin{figure}
    \centering{
\includegraphics[height=4cm]{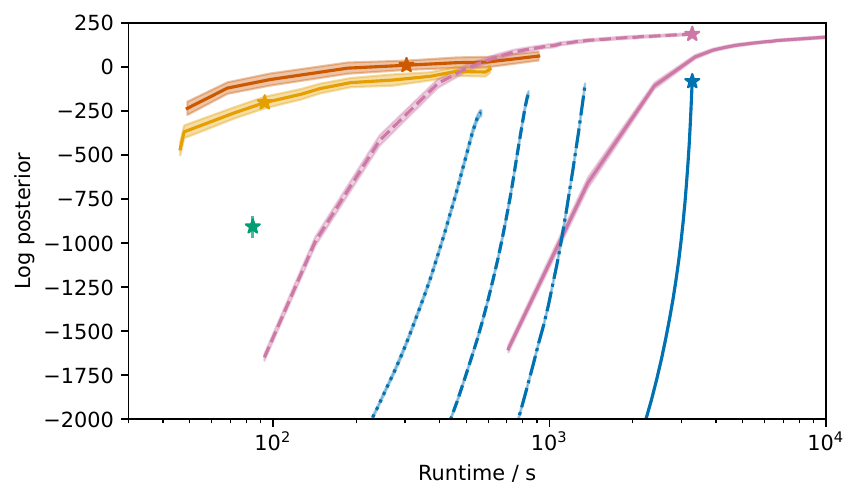}
\includegraphics[height=4cm]{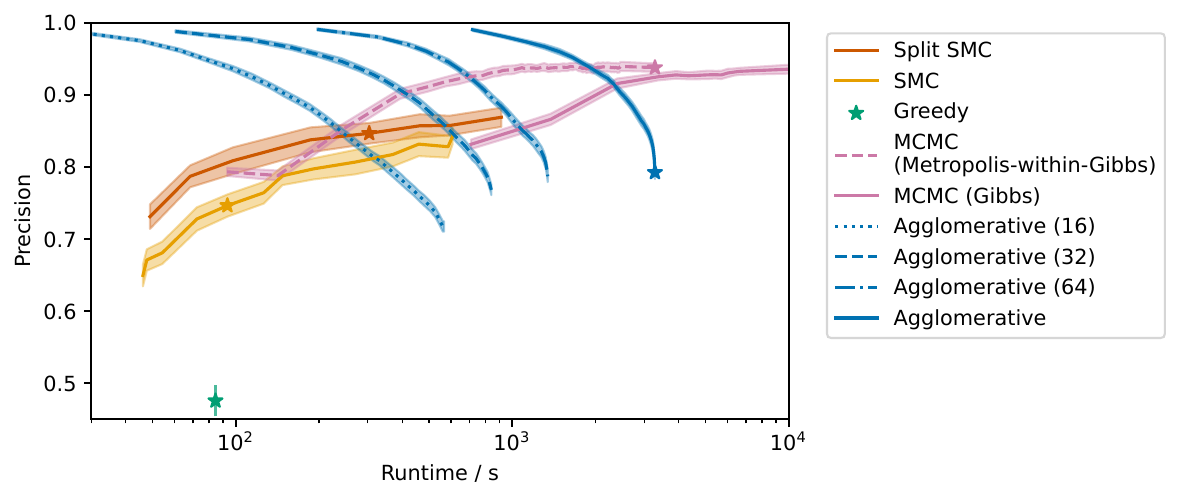}\\
\includegraphics[height=4cm]{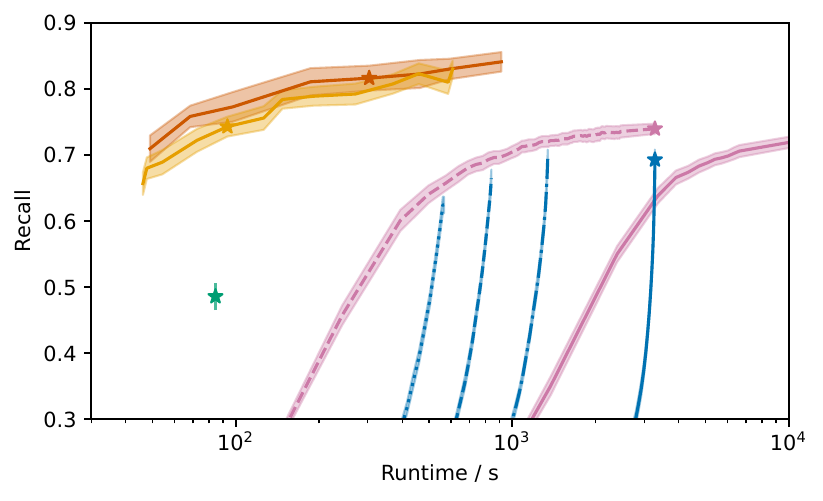}
\includegraphics[height=4cm]{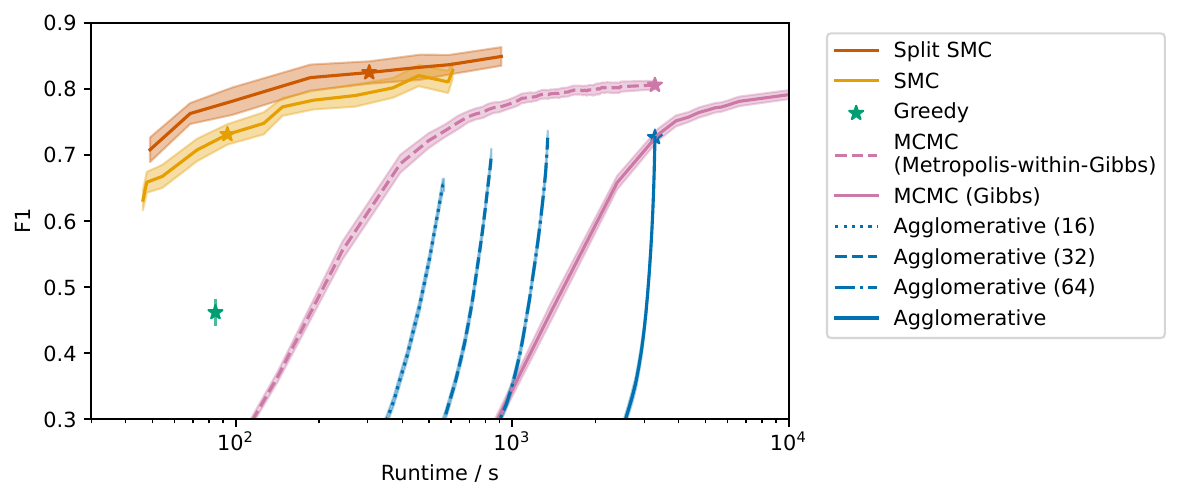}
    }
    \caption{Performance metrics against runtime (log scale) for the circles dataset. The mean across replications is plotted, with a shaded region indicating a 95\% confidence interval. The stars show the configurations reported in Table \ref{tab:metrics}.}
    \label{fig:circles-runtime}
\end{figure}
\begin{figure}
    \centering{
\includegraphics[height=4cm]{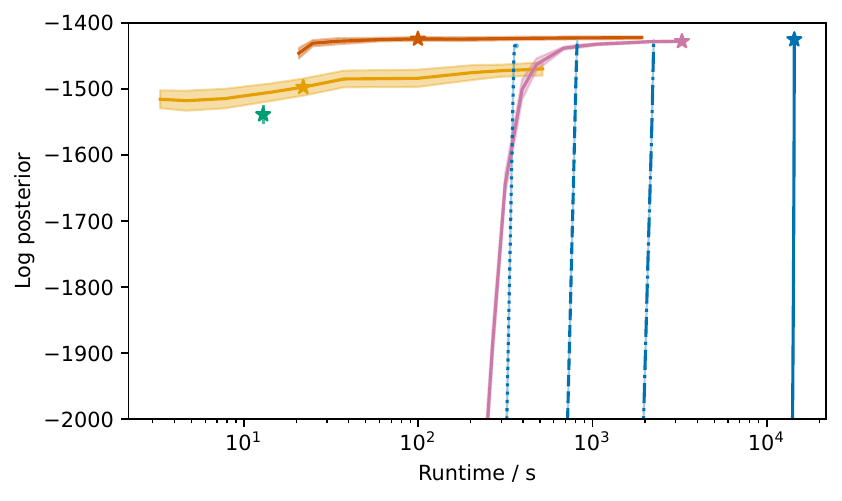}
\includegraphics[height=4cm]{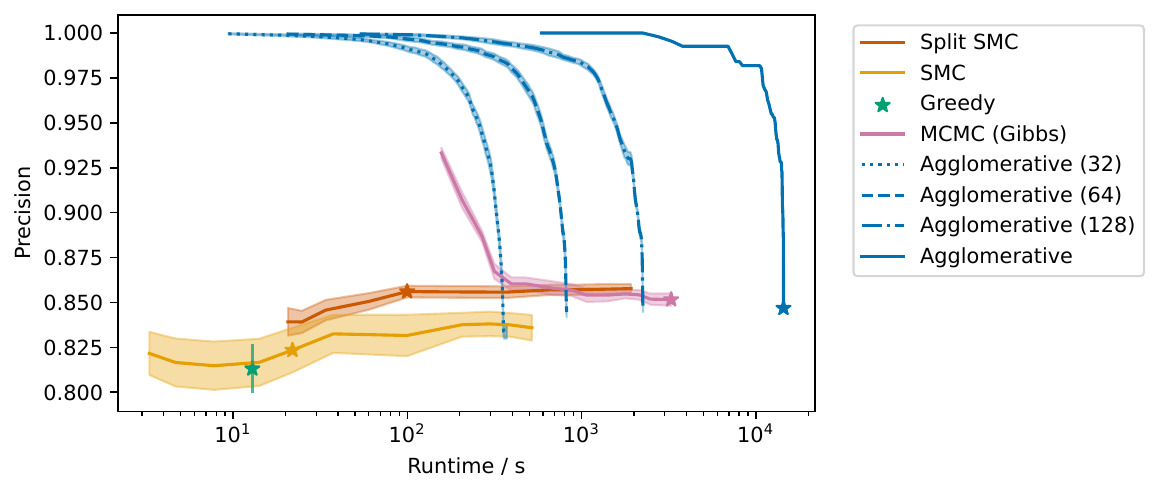}\\
\includegraphics[height=4cm]{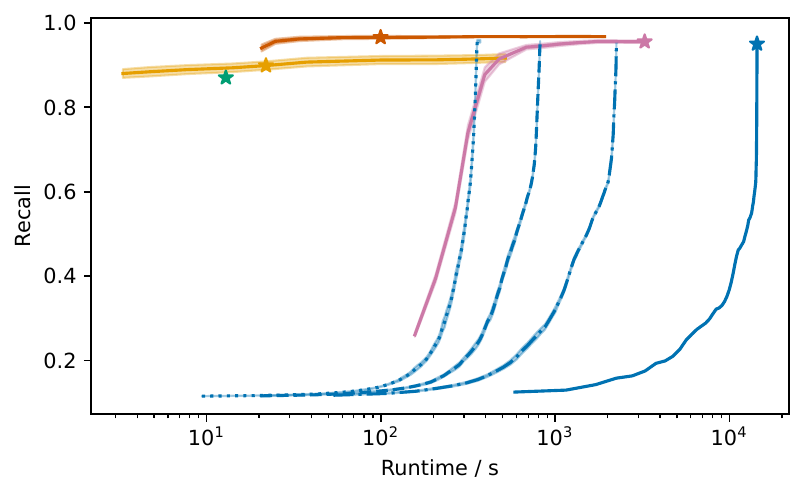}
\includegraphics[height=4cm]{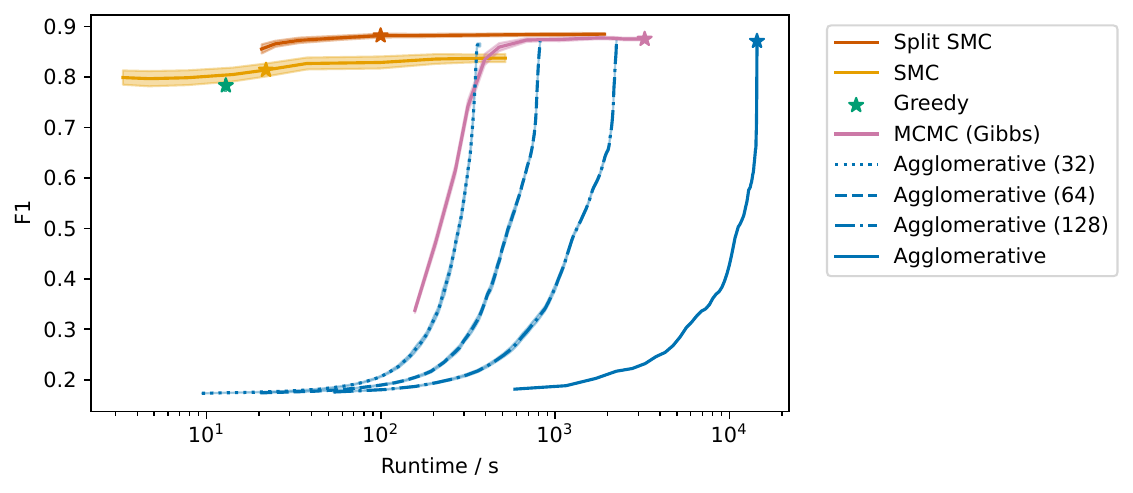}
    }
    \caption{Performance metrics against runtime (log scale) for the GMM dataset. The mean across replications is plotted, with a shaded region indicating a 95\% confidence interval. The stars show the configurations reported in Table \ref{tab:metrics}.}
    \label{fig:gmm-runtime}
\end{figure}
\begin{figure}
    \centering{
\includegraphics[height=4cm]{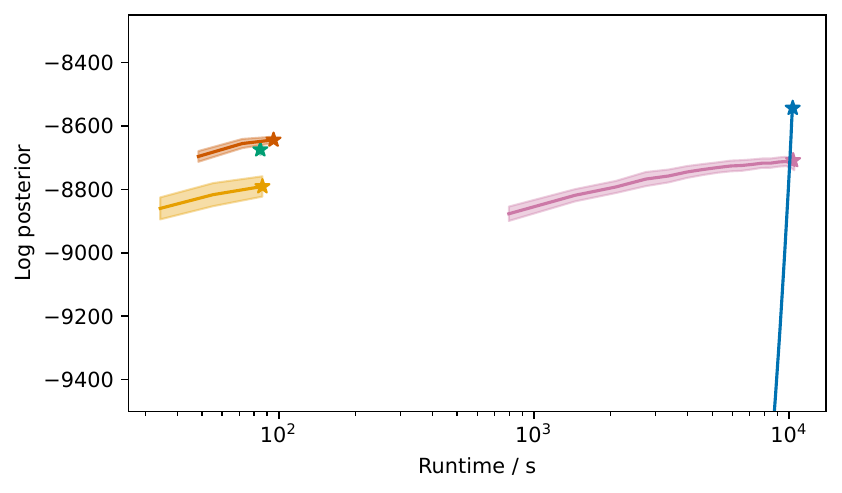}
\includegraphics[height=4cm]{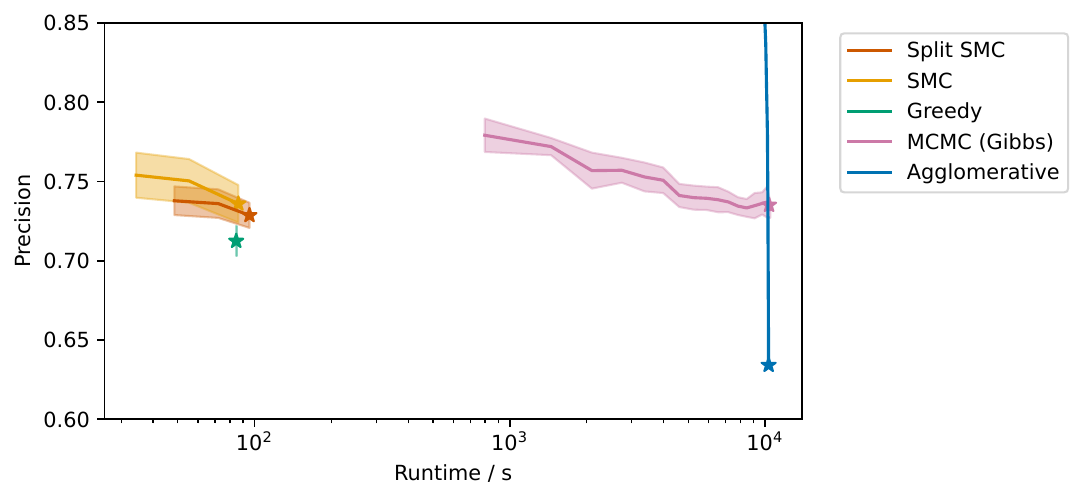}\\
\includegraphics[height=4cm]{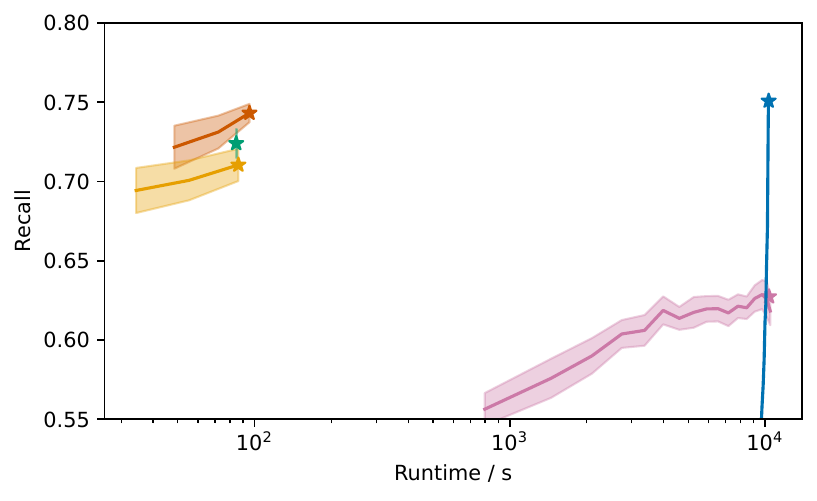}
\includegraphics[height=4cm]{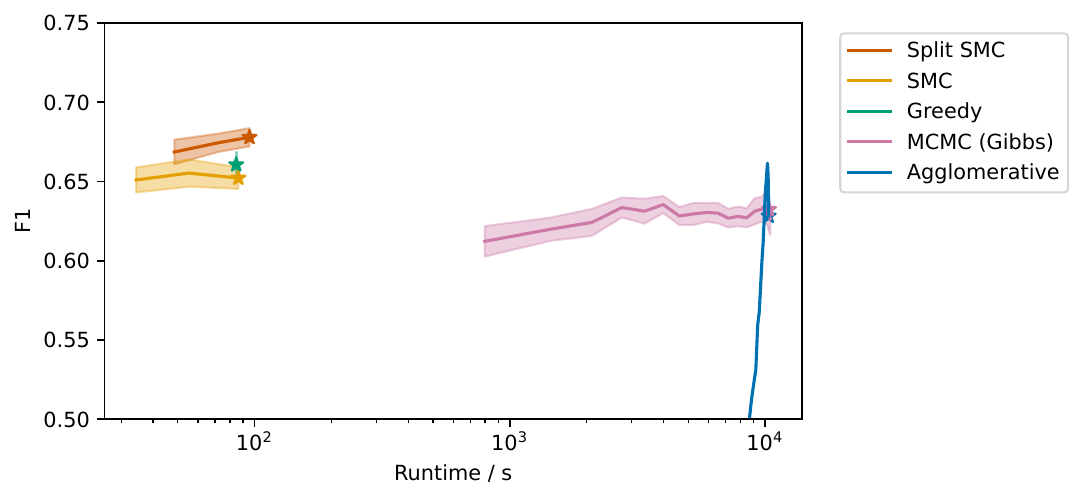}
    }
    \caption{Performance metrics against runtime (log scale) for the REBEL-50 dataset. The mean across replications is plotted, with a shaded region indicating a 95\% confidence interval. The stars show the configurations reported in Table \ref{tab:metrics}.}
    \label{fig:rebel50-runtime}
\end{figure}
\begin{figure}
    \centering{
\includegraphics[height=4cm]{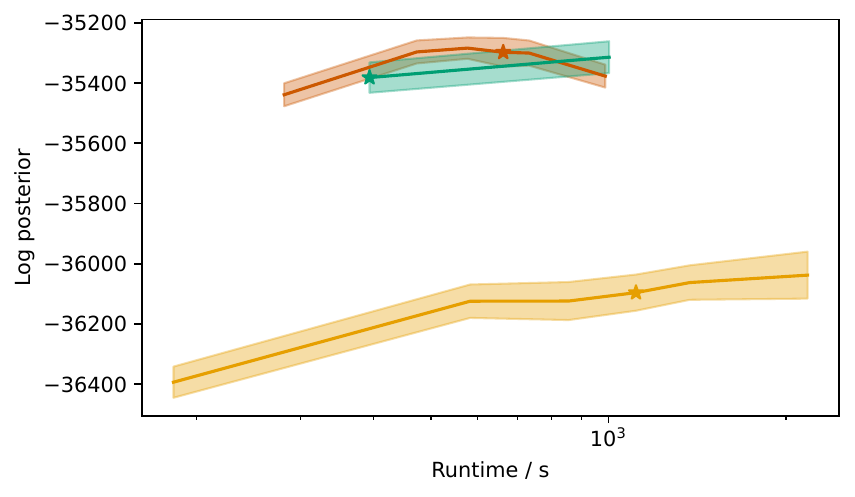}
\includegraphics[height=4cm]{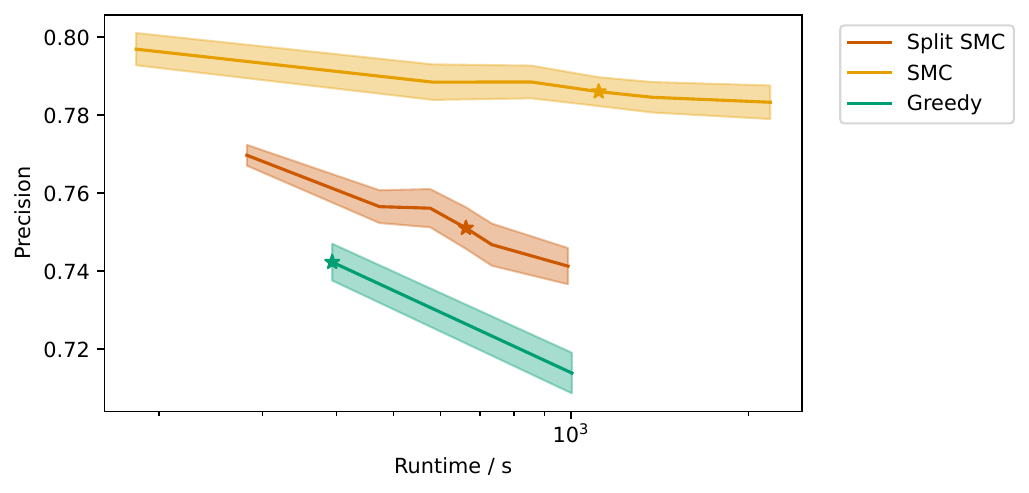}\\
\includegraphics[height=4cm]{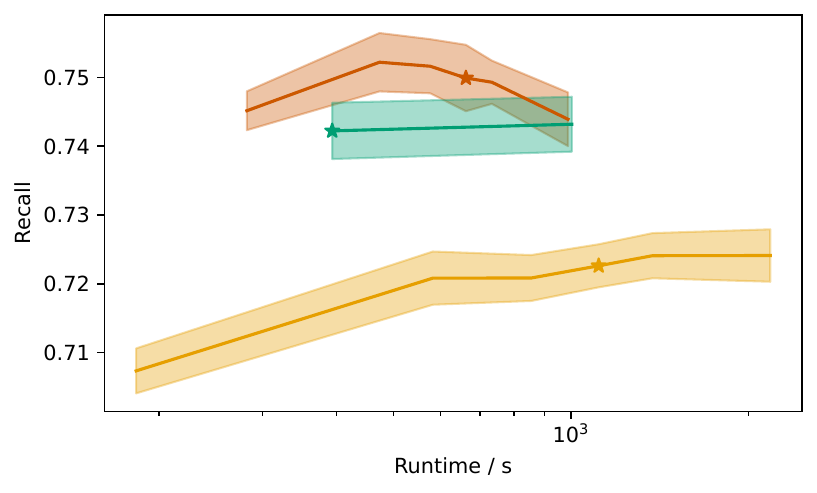}
\includegraphics[height=4cm]{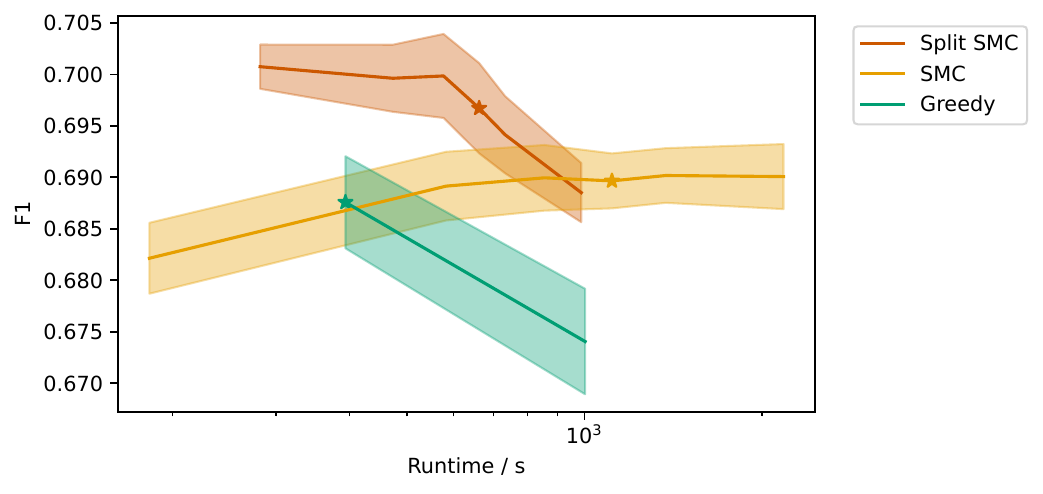}
    }
    \caption{Performance metrics against runtime (log scale) for the REBEL-200 dataset. The mean across replications is plotted, with a shaded region indicating a 95\% confidence interval. The stars show the configurations reported in Table \ref{tab:metrics}.}
    \label{fig:rebel200-runtime}
\end{figure}
\begin{figure}
    \centering{
\includegraphics[height=4cm]{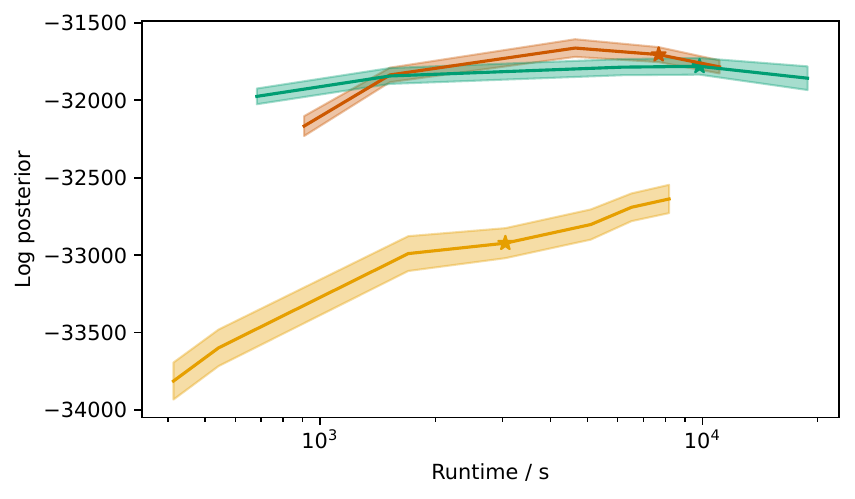}
\includegraphics[height=4cm]{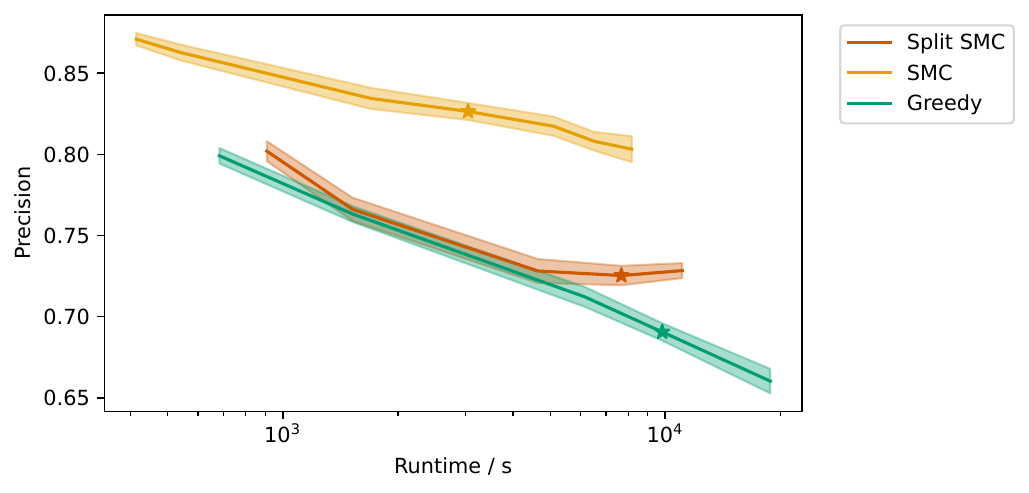}\\
\includegraphics[height=4cm]{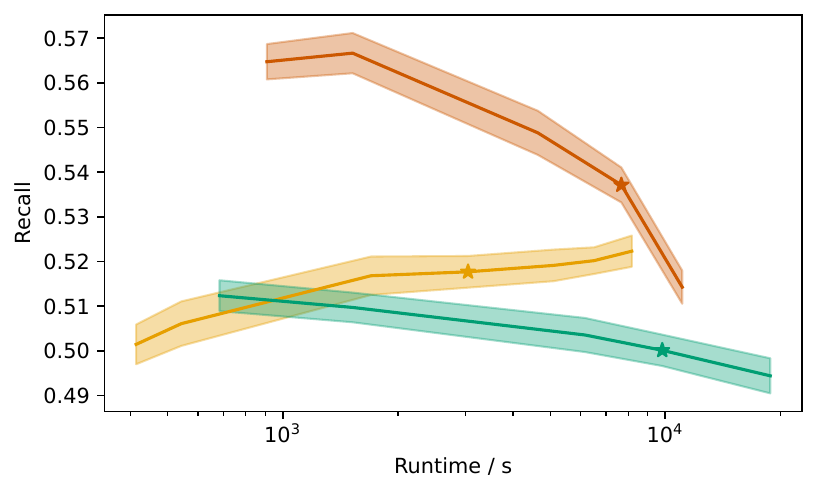}
\includegraphics[height=4cm]{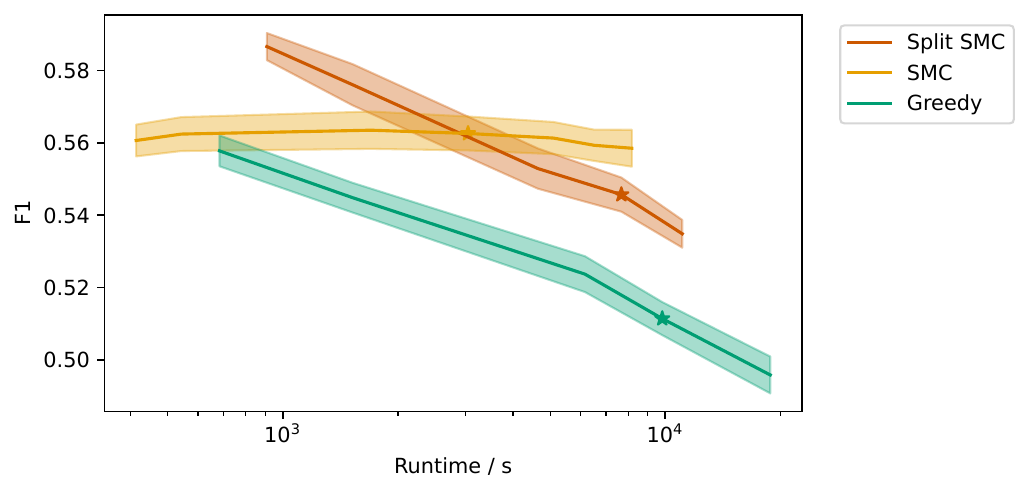}
    }
    \caption{Performance metrics against runtime (log scale) for the TweetNERD dataset. The mean across replications is plotted, with a shaded region indicating a 95\% confidence interval. The stars show the configurations reported in Table \ref{tab:metrics}.}
    \label{fig:tweet-runtime}
\end{figure}

\pagebreak
\subsection{Particle Set Size Comparison}\label{sec:particles}

Figures \labelcref{fig:circles-particles,fig:gauss-particles,fig:rebel50-particles,fig:rebel200-particles,fig:tweet-particles} plot each metric against particle set size for the SMC algorithms. When using neural likelihood models on CPU, split SMC had a higher runtime than vanilla SMC for any given particle set size $m$, so we ran vanilla SMC for a larger range of $m$ to enable a fair comparison on runtime.

\begin{figure}
    \centering{
\includegraphics[height=4cm]{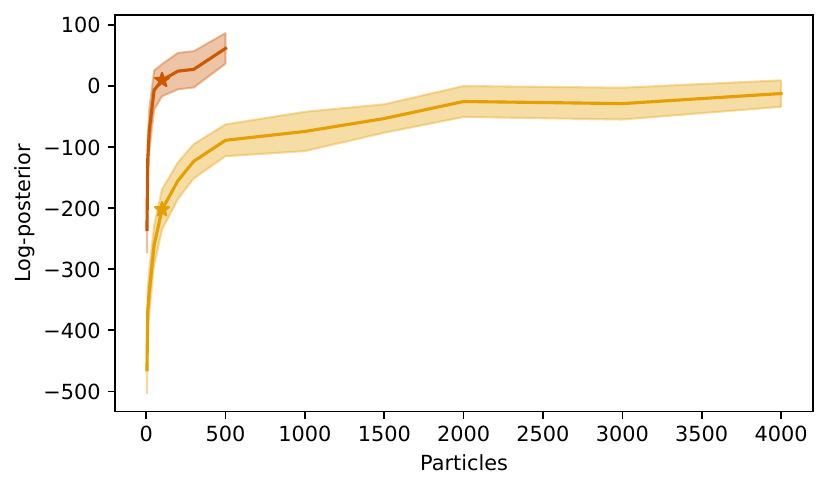}
\includegraphics[height=4cm]{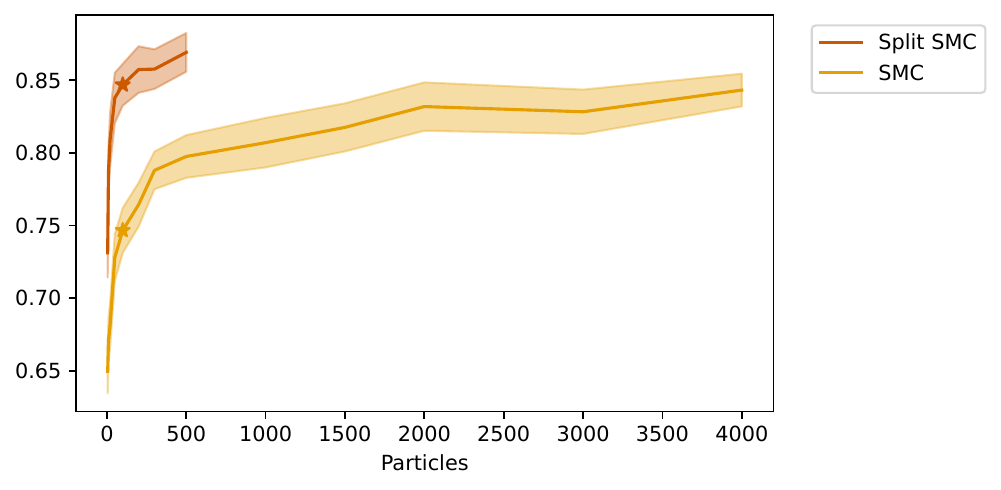}\\
\includegraphics[height=4cm]{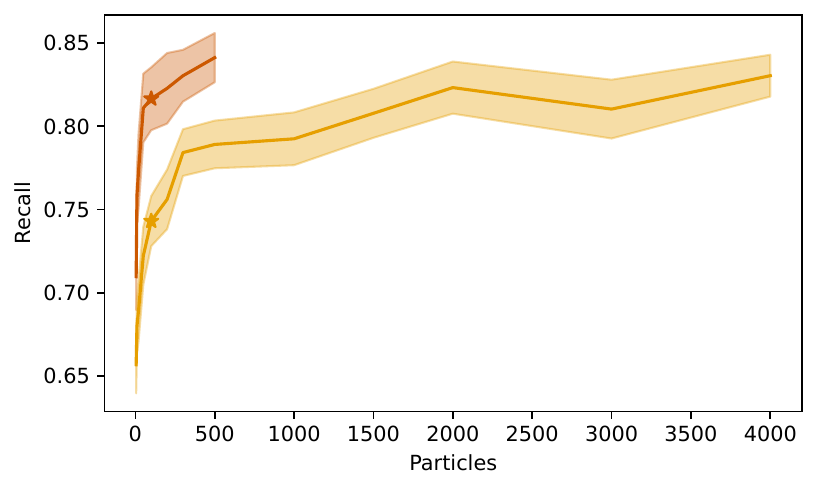}
\includegraphics[height=4cm]{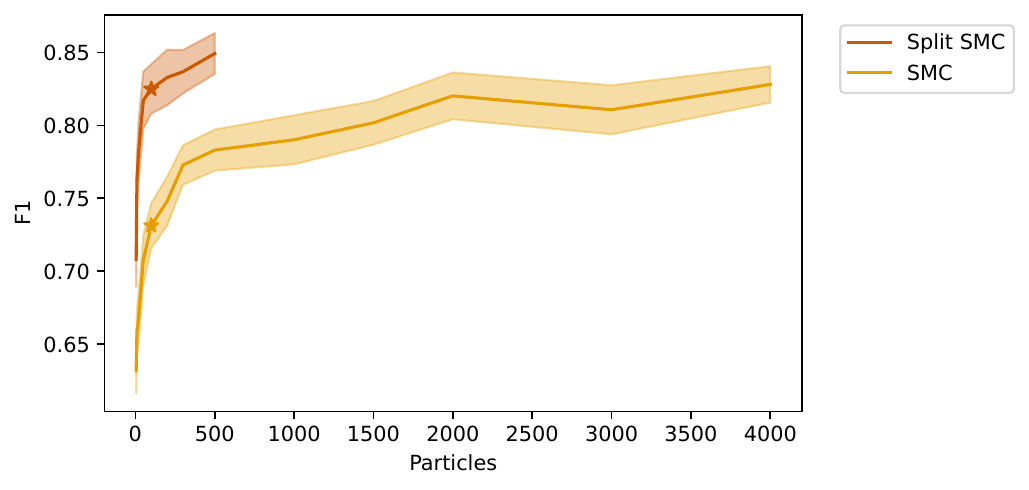}
    }
    \caption{Performance metrics against particle set size for the circles dataset. The mean across replications is plotted, with a shaded region indicating a 95\% confidence interval. The stars show the configurations reported in Table \ref{tab:metrics}.}
    \label{fig:circles-particles}
\end{figure}
\begin{figure}
    \centering{
\includegraphics[height=4cm]{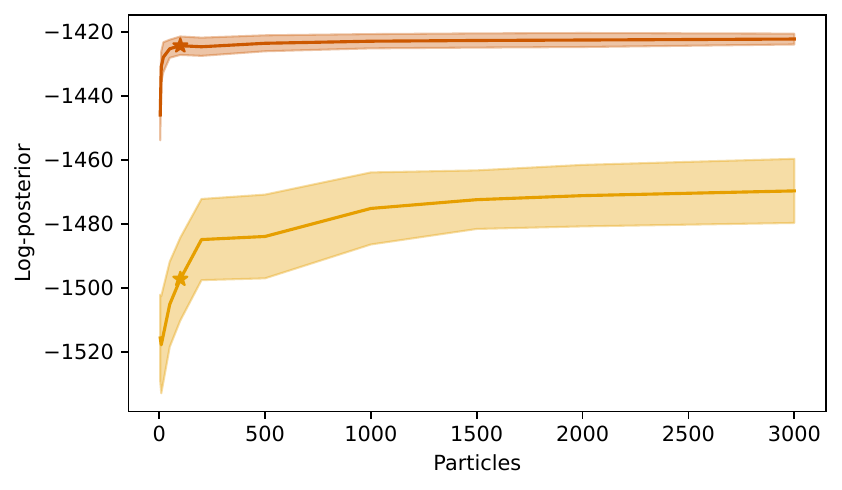}
\includegraphics[height=4cm]{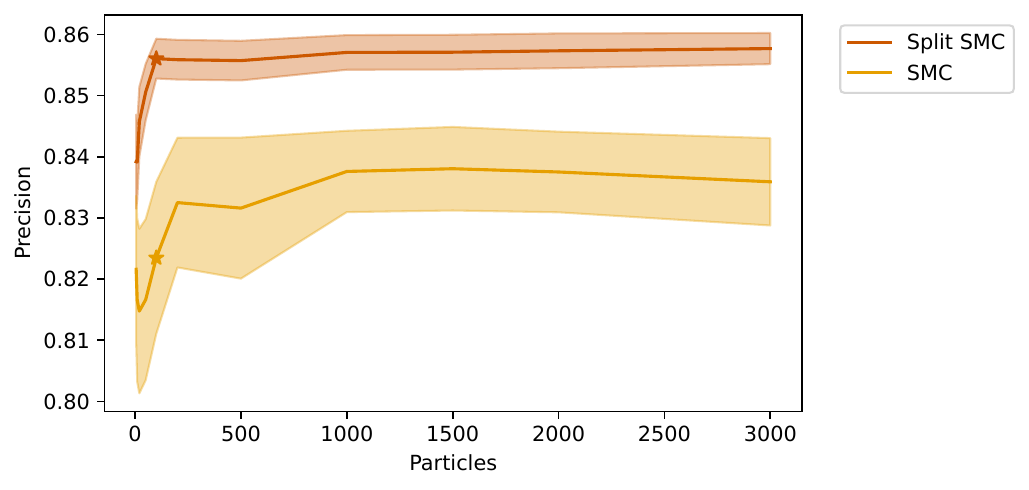}\\
\includegraphics[height=4cm]{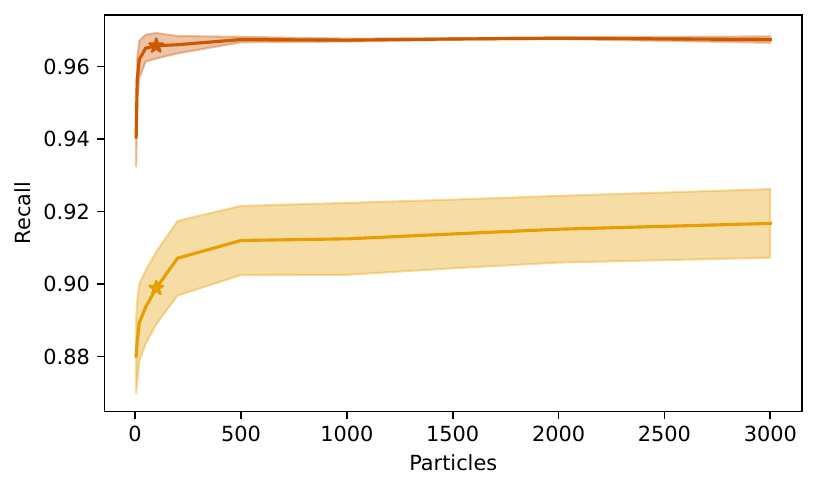}
\includegraphics[height=4cm]{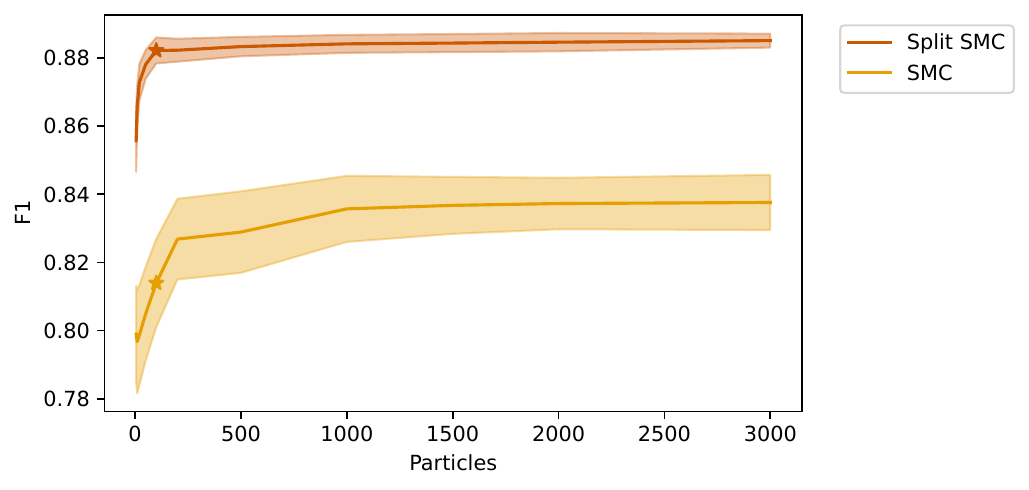}
    }
    \caption{Performance metrics against particle set size for the GMM dataset. The mean across replications is plotted, with a shaded region indicating a 95\% confidence interval. The stars show the configurations reported in Table \ref{tab:metrics}.}
    \label{fig:gauss-particles}
\end{figure}

\begin{figure}
    \centering{
\includegraphics[height=4cm]{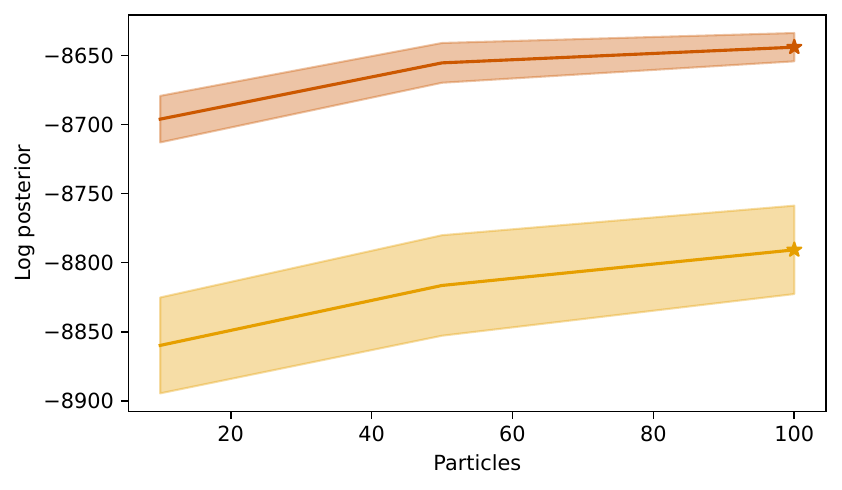}
\includegraphics[height=4cm]{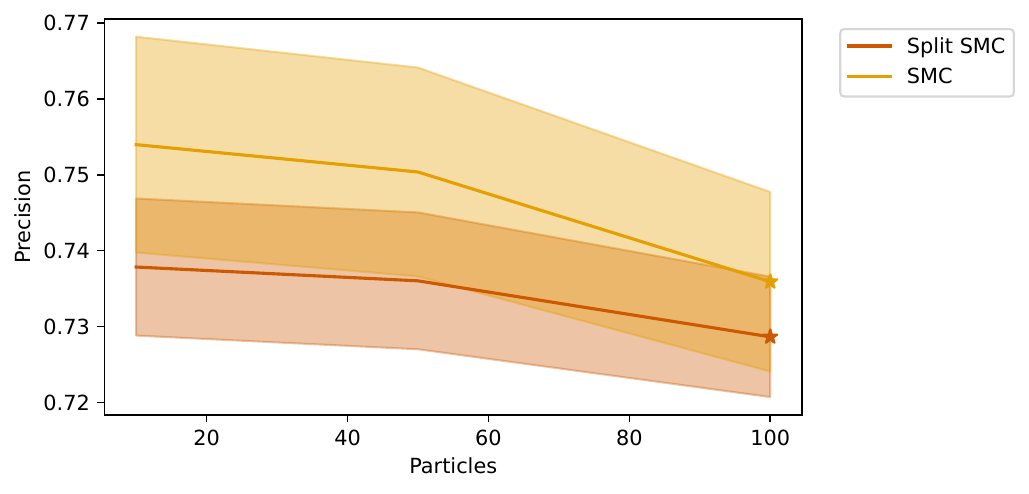}\\
\includegraphics[height=4cm]{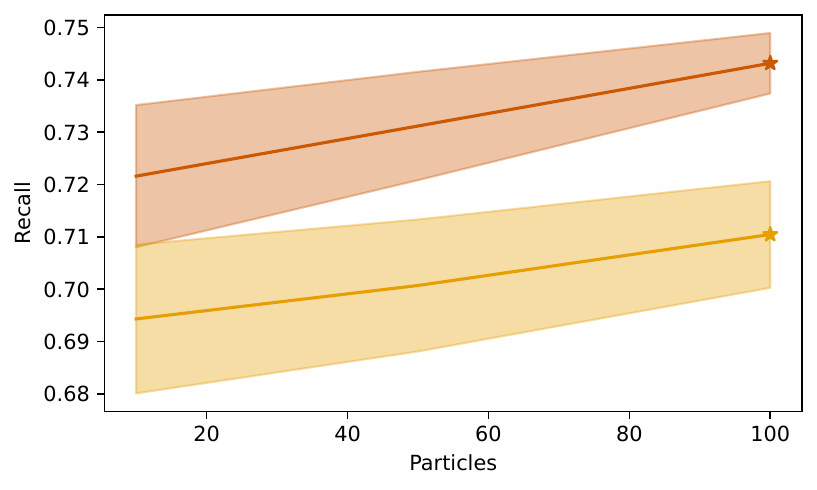}
\includegraphics[height=4cm]{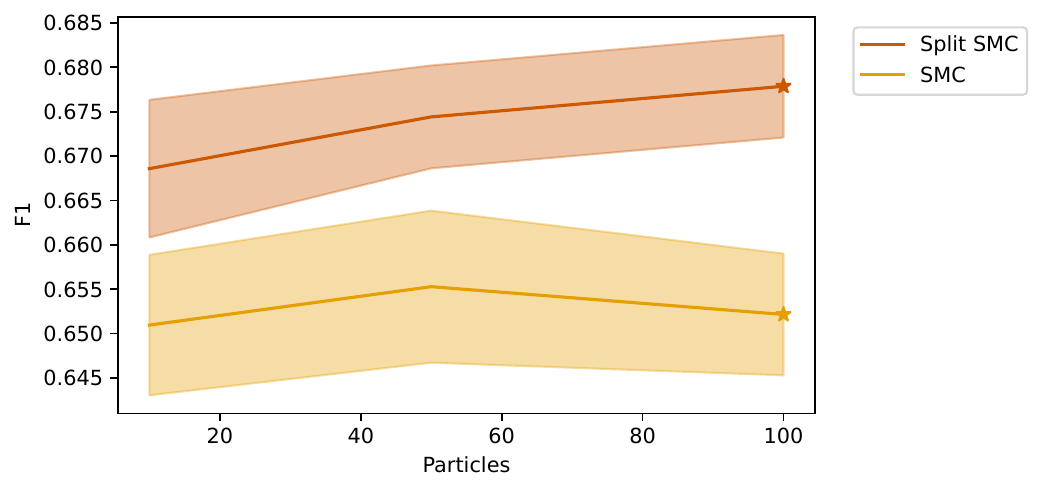}
    }
    \caption{Performance metrics against particle set size for the REBEL-50 dataset. The mean across replications is plotted, with a shaded region indicating a 95\% confidence interval. The stars show the configurations reported in Table \ref{tab:metrics}.}
    \label{fig:rebel50-particles}
\end{figure}

\begin{figure}
    \centering{
\includegraphics[height=4cm]{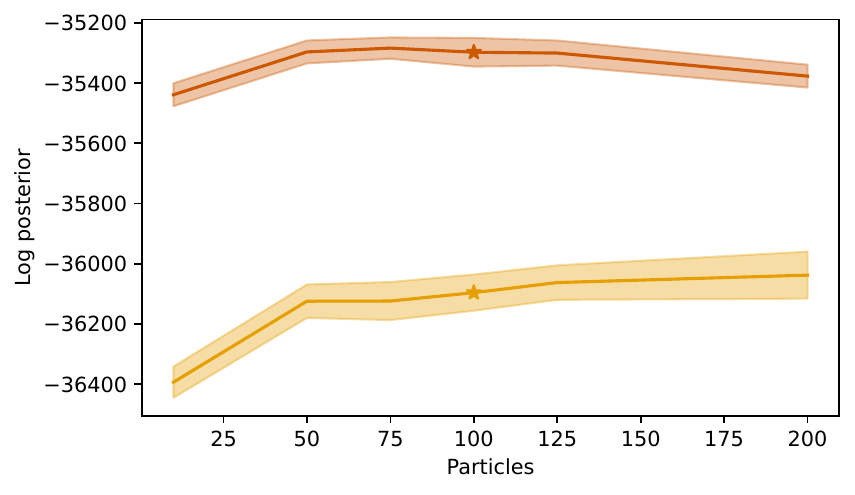}
\includegraphics[height=4cm]{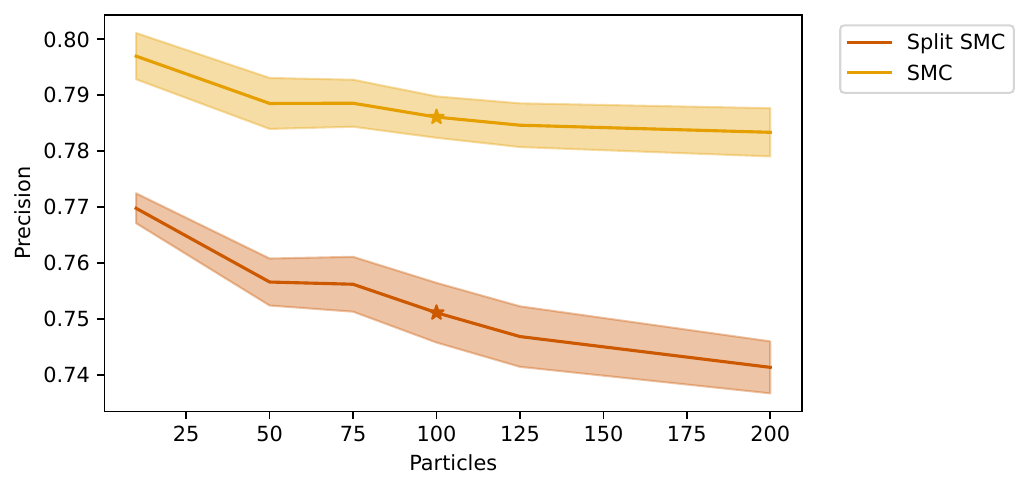}\\
\includegraphics[height=4cm]{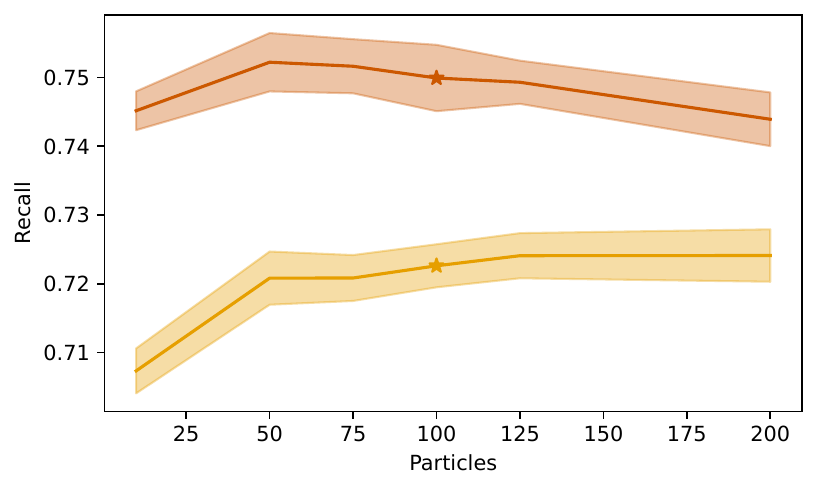}
\includegraphics[height=4cm]{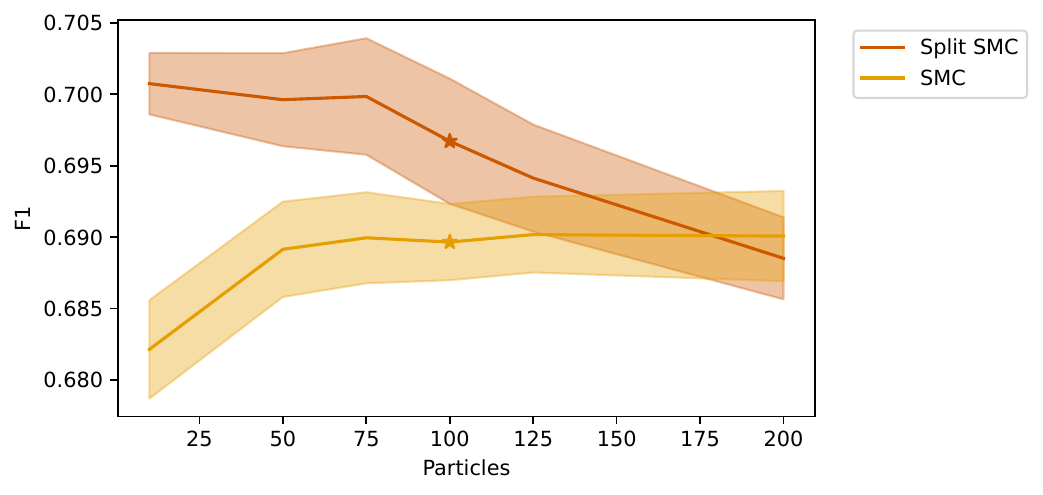}
    }
    \caption{Performance metrics against particle set size for the REBEL-200 dataset. The mean across replications is plotted, with a shaded region indicating a 95\% confidence interval. The stars show the configurations reported in Table \ref{tab:metrics}.}
    \label{fig:rebel200-particles}
\end{figure}

\begin{figure}
    \centering{
\includegraphics[height=4cm]{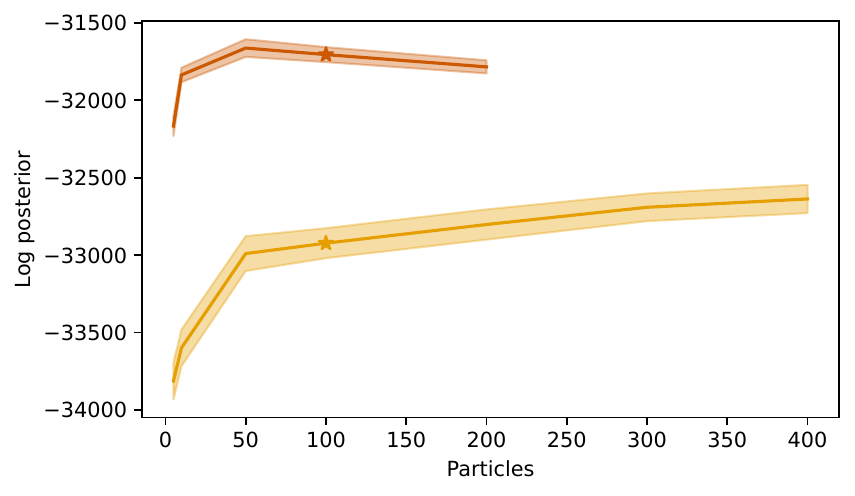}
\includegraphics[height=4cm]{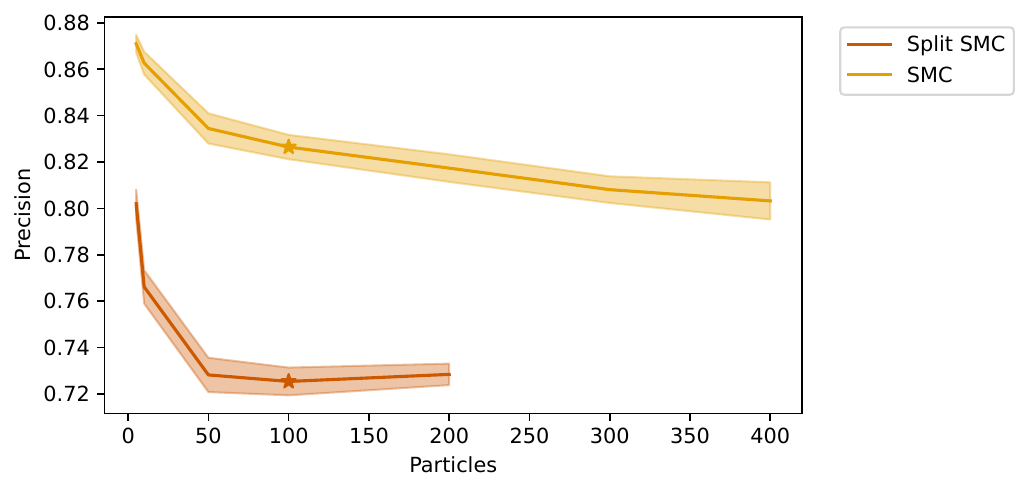}\\
\includegraphics[height=4cm]{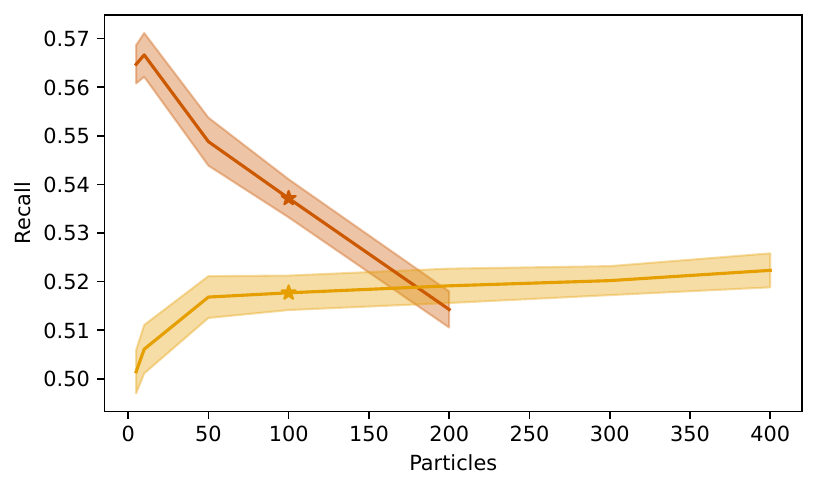}
\includegraphics[height=4cm]{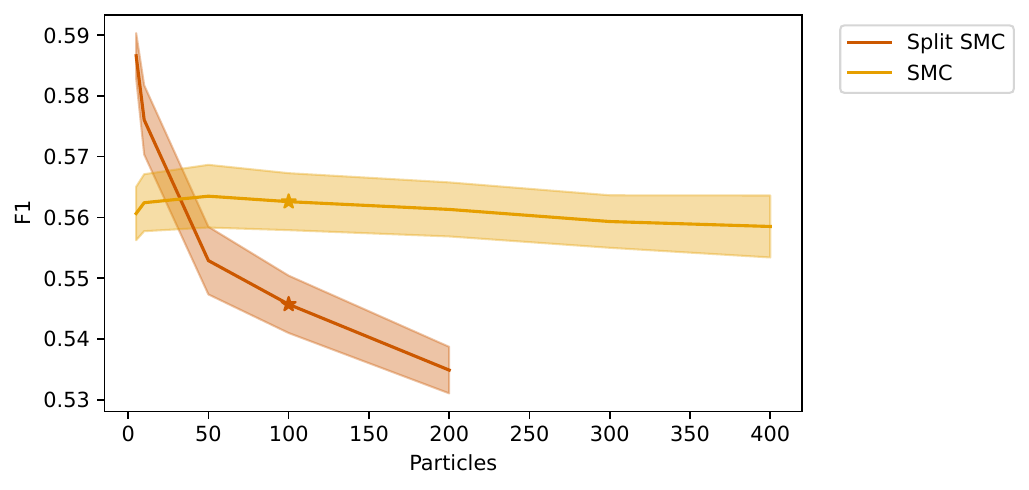}
    }
    \caption{Performance metrics against particle set size for the TweetNERD dataset. The mean across replications is plotted, with a shaded region indicating a 95\% confidence interval. The stars show the configurations reported in Table \ref{tab:metrics}.}
    \label{fig:tweet-particles}
\end{figure}

\subsubsection{Effective Particle Set Size in Split SMC}\label{sec:ess}

The number of particles implicitly represented by the split particle set can be computed by taking the product of the sizes of each subproblem. We plot the number of subproblems and the log of the effective number of particles for each experiment in Figure \ref{fig:subprobs-particles}. The number of subproblems is decreasing in $m$ in each case, since larger values of $m$ allow us to store lower-probability clusterings that introduce overlap between potential partitions of the dataset. This decrease is dramatic for smaller values of $m$, but levels off as $m$ increases. The effective number of particles typically increases with $m$ regardless. However, on the text datasets, it decreases again for larger values: this may explain why increasing $m$ past 100 did not improve the log-posterior density in these experiments.

\begin{figure}
    \centering{
\includegraphics[height=4cm]{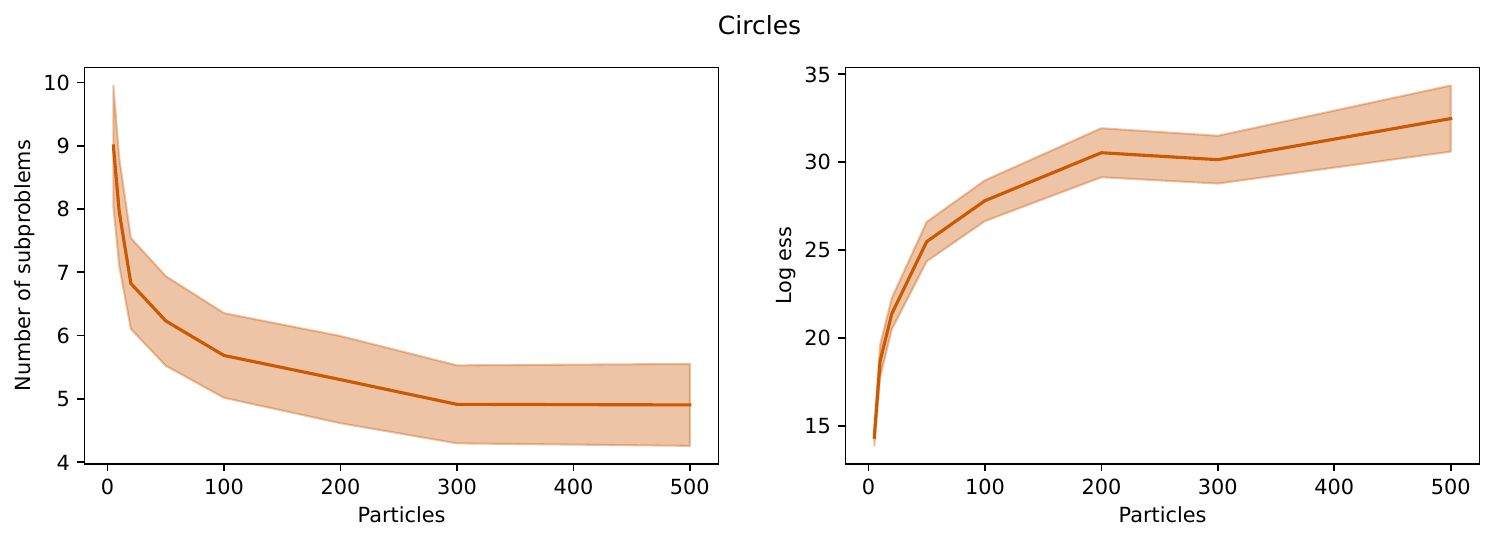}
\includegraphics[height=4cm]{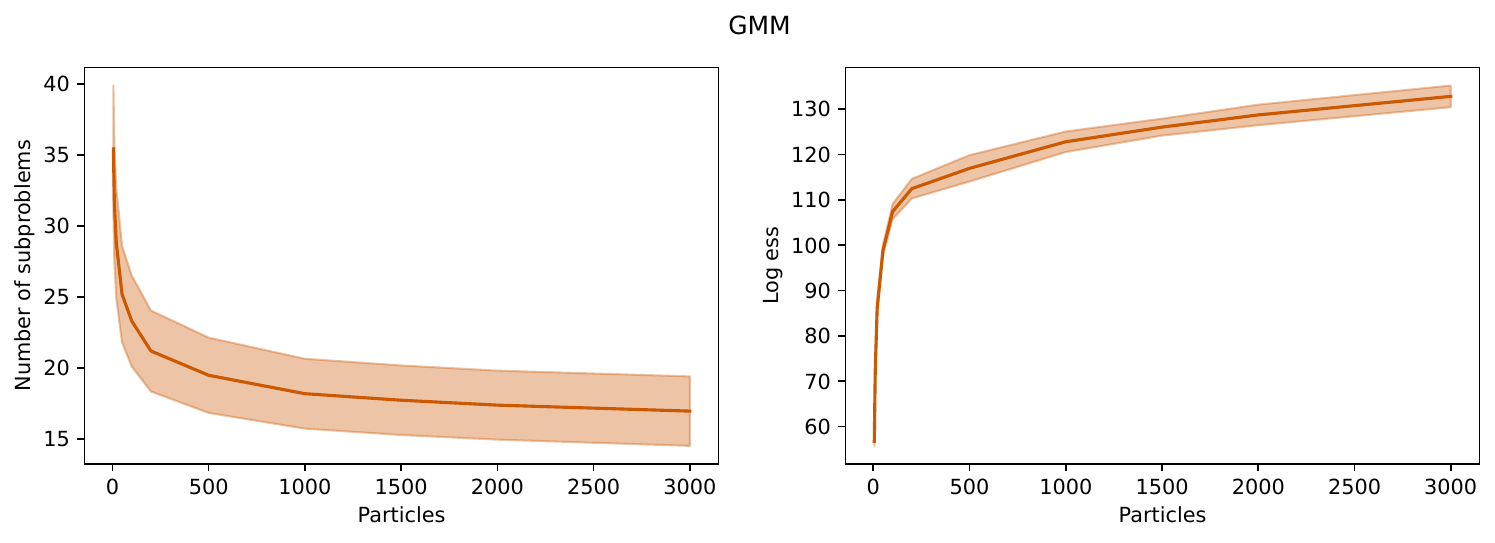}\\
\includegraphics[height=4cm]{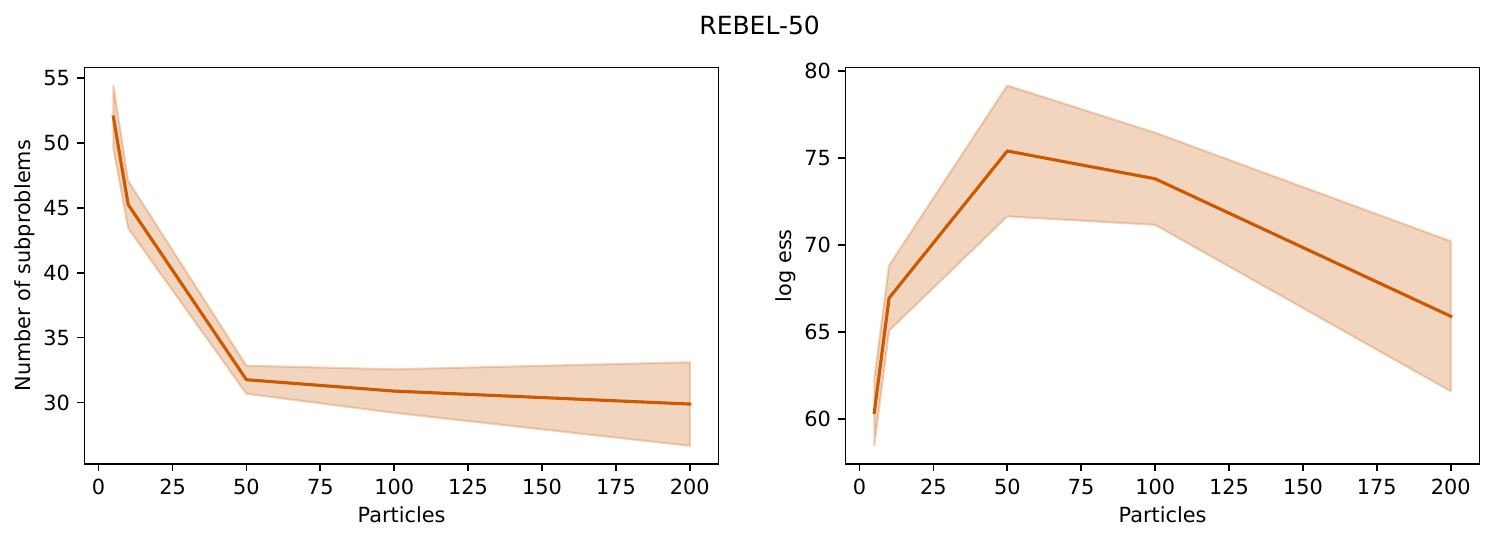}
\includegraphics[height=4cm]{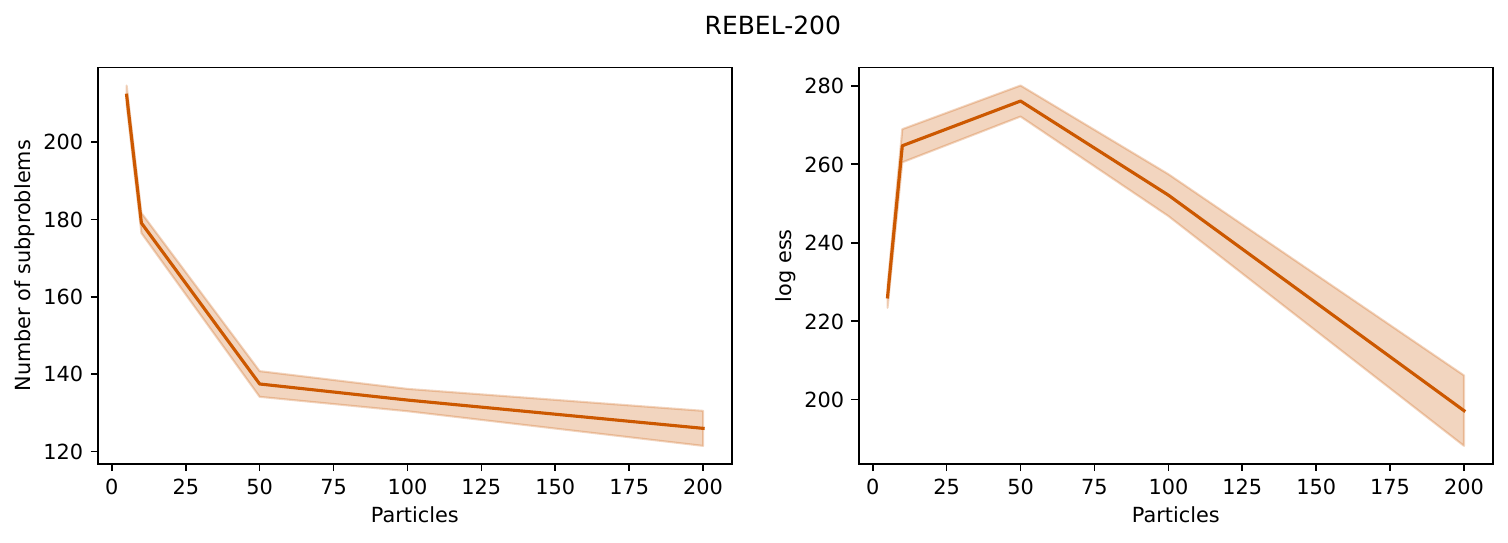}
\includegraphics[height=4cm]{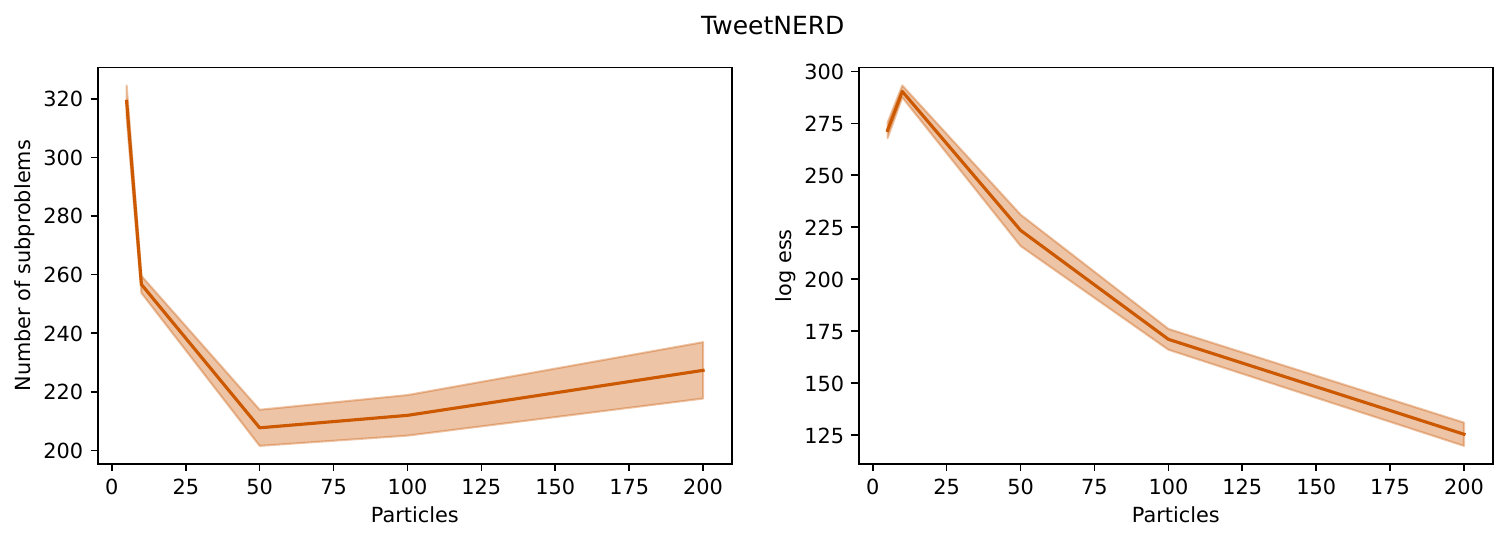}
    }
    \caption{Number of subproblems (left) and log effective particle set size (right) against particle set size $m$ for each dataset. The mean across replications is plotted, with a shaded region indicating a 95\% confidence interval.}
    \label{fig:subprobs-particles}
\end{figure}

\FloatBarrier

\subsection{Surrogate Model Ablation}\label{sec:surrogate-ablation}

We investigate the impact of using surrogate models by fixing the number of particles used in sequential clustering and varying the number of model evaluations made. This was done by changing the number of particles proposed under the surrogate likelihood, which are then re-weighted and resampled according to the model likelihood. We call this parameter $m'$, and a detailed explanation of the surrogate proposal step can be found in Appendix \ref{sec:smc-details}. In the limit as the parameter $m'$ increases, the surrogate model is not used at all. We used $m'=m$ in all previous experiments, with the exception of the greedy algorithm where this was set according to evaluation budget (not used for 2D data, and $m'=100$ reported for text data). In each of the experiments below, we found that the clustering with the highest F1 was found by split SMC with $m'=m$, and that split SMC found the clusterings with the highest log-posterior density for each $m'$.

Figure \ref{fig:circles-evals} shows results for $m=50$ on the circles data. Surprisingly, increasing the number of model evaluations actually decreases the log-posterior density of the clusterings found for each algorithm, and has a negative impact on every accuracy metric. The dashed horizontal lines on the plot indicate results obtained with $m=50$ using the surrogate model alone -- we can see that the surrogate model results have much higher recall but lower precision compared to those obtained with the neural likelihood model. The F1 obtained with the surrogate is lower than that of split SMC with $m'=m=50$, but greater than or equal to those of every other configuration. 

Figures \labelcref{fig:rebel50-evals,fig:rebel200-evals,fig:tweet-evals} show the results for $m=10$ on the REBEL-50, REBEL-200 and TweetNERD data. This time, increasing the number of model evaluations increases log-posterior density, indicating a greater mismatch between the surrogate and neural models. On TweetNERD, log-posterior density starts to decline again as $m'$ increases. On both REBEL datasets, precision decreases and recall increases in $m'$ for all methods. On REBEL-200, F1 decreases with $m'$ for split SMC and greedy, while for vanilla SMC it increases a little and then starts to decline again slowly.

On TweetNERD, accuracy metrics are decreasing in $m'$ for all methods, despite improvements in log-posterior density. This mirrors the results in Figure \ref{fig:tweet-particles}, where the accuracy of split SMC is decreasing in the number of particles. This may be due to the fact that, on this problem, the surrogate model produced more accurate clusterings when used alone than when the neural model was used. This is likely because the TweetNERD dataset was very out-of-domain for the model, which was trained on REBEL. TweetNERD contained fragments with the name attribute only, and these observations had considerably more within-cluster variability than the REBEL training data. Nevertheless, increased use of the surrogate model allowed algorithms to find clusterings with similar log-posterior density but higher accuracy.

These results indicate that using a surrogate model not only helps reduce runtime by limiting model evaluations, but is also beneficial for accuracy and can even help find clusterings with higher model likelihood in some problems. This effect was not observed with MCMC: Figure \ref{fig:circles-runtime} shows that on the circles data, while there is a substantial improvement in runtime, there is minimal difference between results obtained with normal Gibbs sampling and Metropolis-within-Gibbs, where the surrogate is similarly used as a proposal distribution to reduce model evaluations. This may be due to the challenges of clustering with a complex likelihood function in the online setting -- over-optimising model likelihood early in clustering may cause algorithms to discard more accurate clusterings that would later have higher weight after more data is observed. Combining a neural likelihood with a good choice of surrogate model likely has a regularising effect, as clusterings with low surrogate likelihood are discarded by the proposal step regardless of their likelihood under the model. The surrogate models used had very good recall, and seemed to discard correct clusterings relatively rarely.

\begin{figure}
    \centering{
\includegraphics[height=3.4cm]{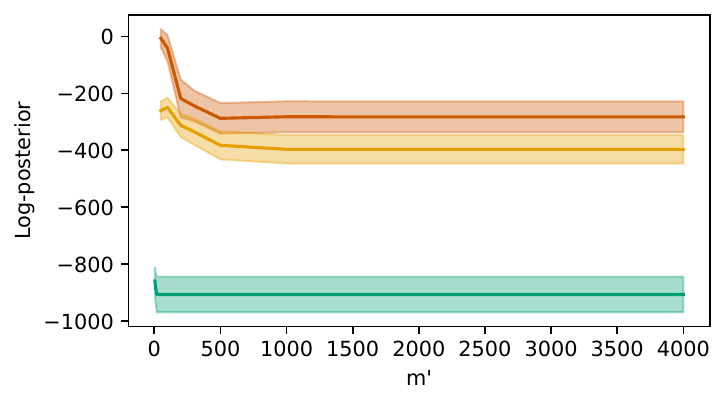}
\includegraphics[height=3.4cm]{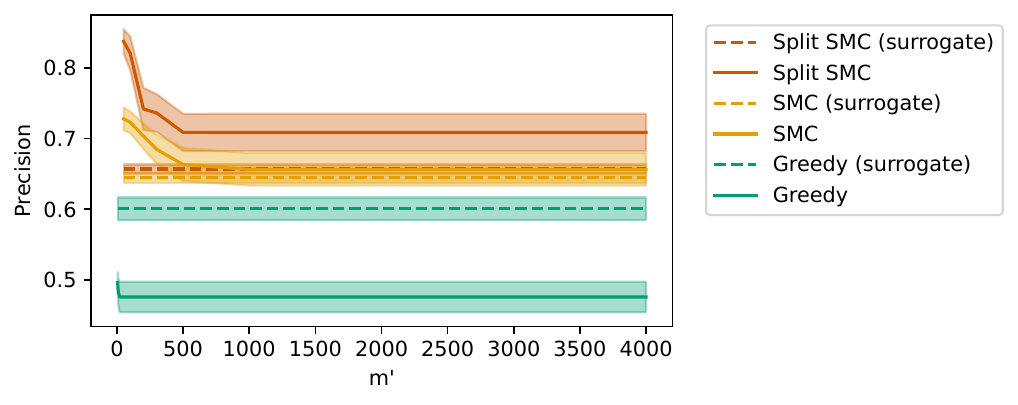}
\includegraphics[height=3.4cm]{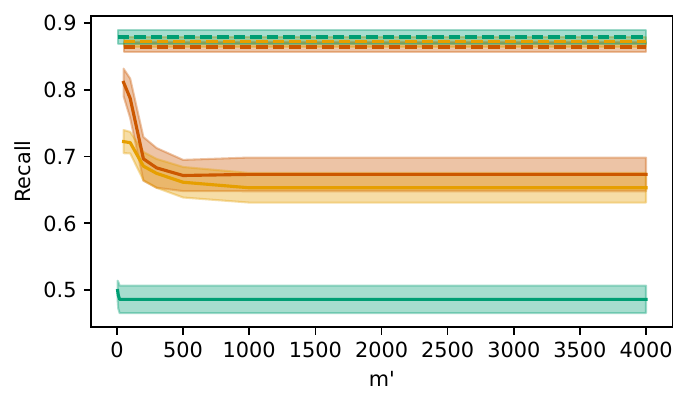}
\includegraphics[height=3.4cm]{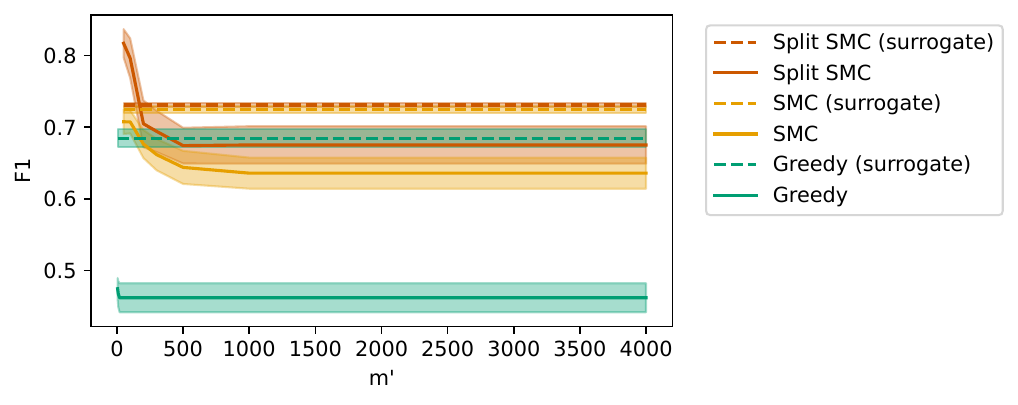}

    }
    \caption{Performance metrics against surrogate proposal size $m'$ on the circles dataset. The mean across replications is plotted, with a shaded region indicating a 95\% confidence interval. Dashed horizontal lines indicate results using the surrogate model alone.}
    \label{fig:circles-evals}
\end{figure}

\begin{figure}
    \centering{
\includegraphics[height=3.4cm]{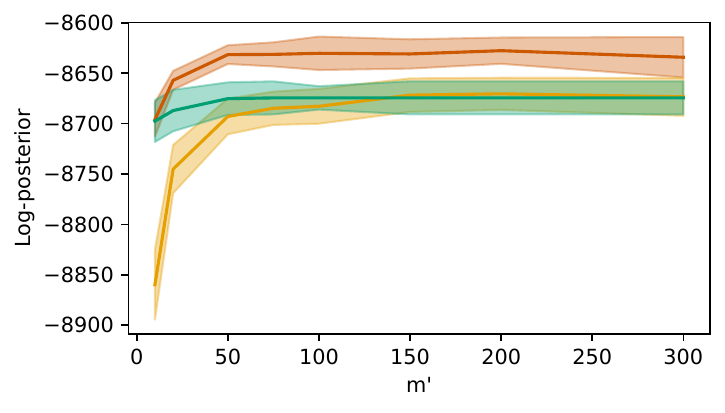}
\includegraphics[height=3.4cm]{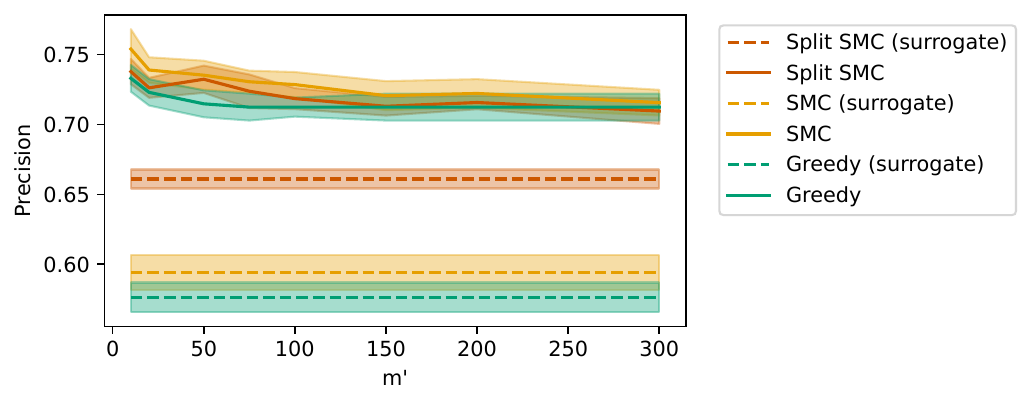}
\includegraphics[height=3.4cm]{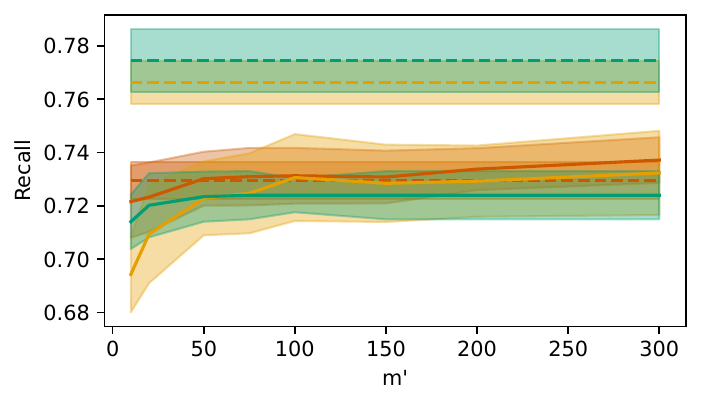}
\includegraphics[height=3.4cm]{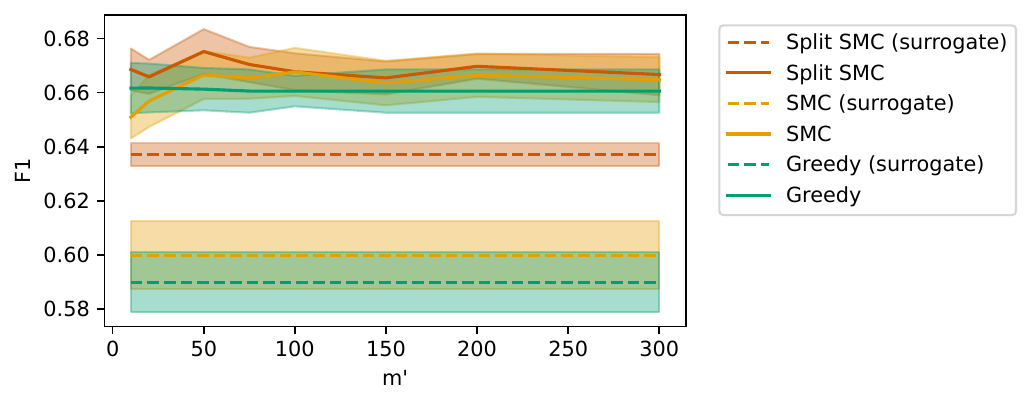}

    }
    \caption{Performance metrics against surrogate proposal size $m'$ on the REBEL-50 dataset. The mean across replications is plotted, with a shaded region indicating a 95\% confidence interval. Dashed horizontal lines indicate results using the surrogate model alone.}
    \label{fig:rebel50-evals}
\end{figure}

\begin{figure}
    \centering{
\includegraphics[height=3.4cm]{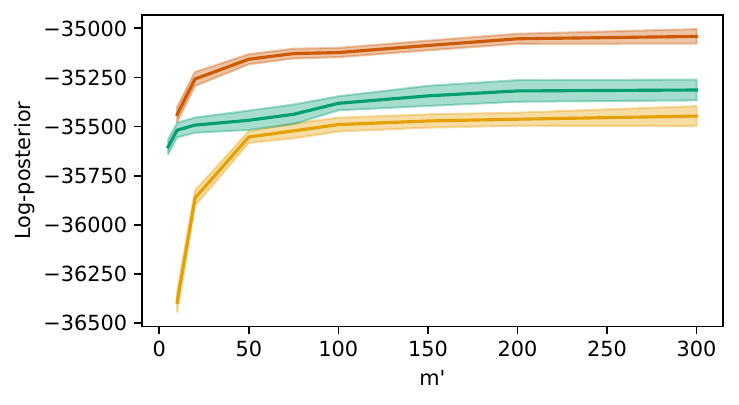}
\includegraphics[height=3.4cm]{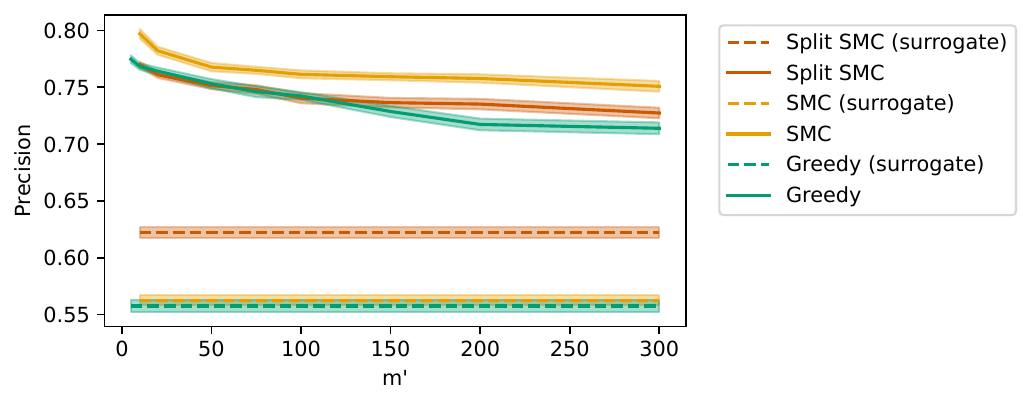}
\includegraphics[height=3.4cm]{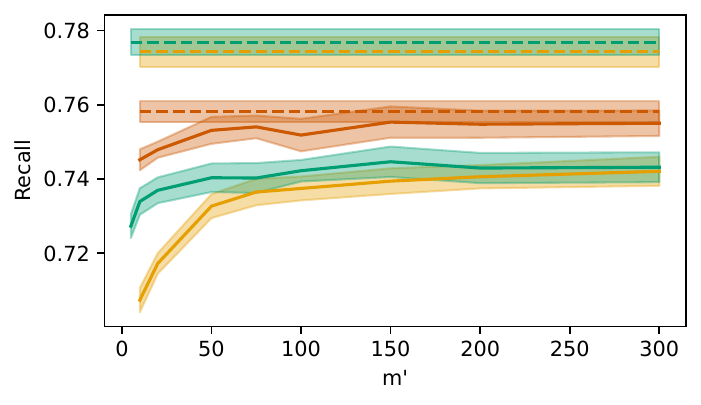}
\includegraphics[height=3.4cm]{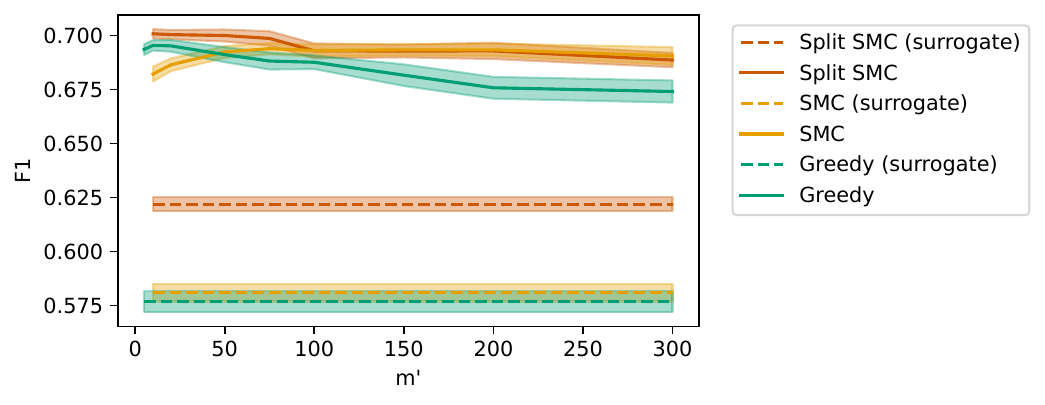}

    }
    \caption{Performance metrics against surrogate proposal size $m'$ on the REBEL-200 dataset. The mean across replications is plotted, with a shaded region indicating a 95\% confidence interval. Dashed horizontal lines indicate results using the surrogate model alone.}
    \label{fig:rebel200-evals}
\end{figure}

\begin{figure}
    \centering{
\includegraphics[height=3.4cm]{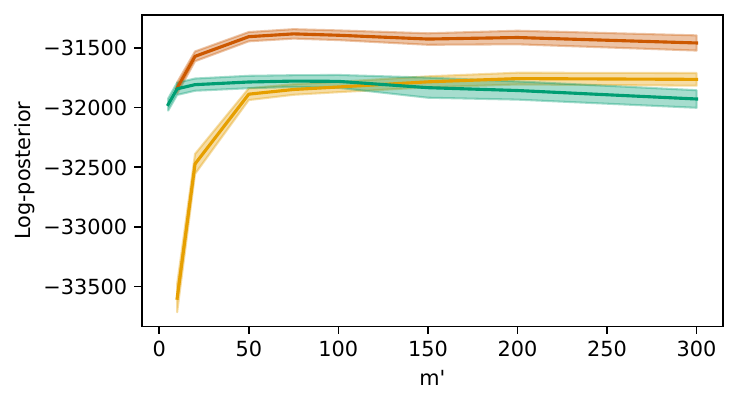}
\includegraphics[height=3.4cm]{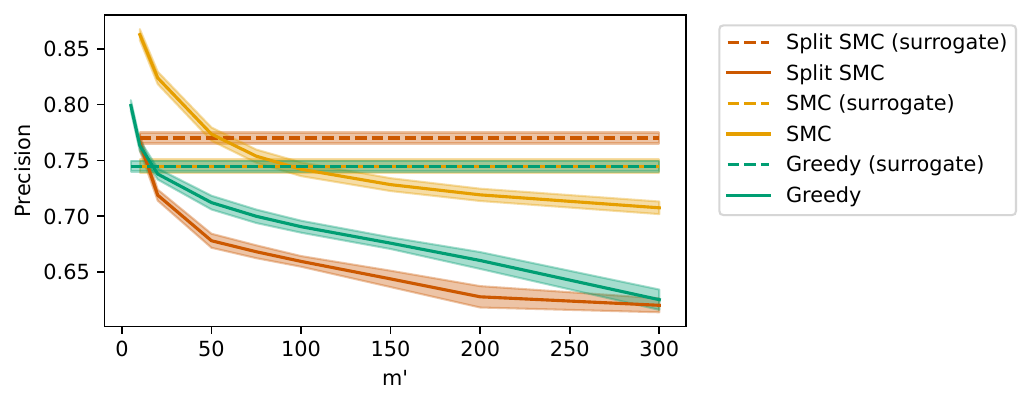}
\includegraphics[height=3.4cm]{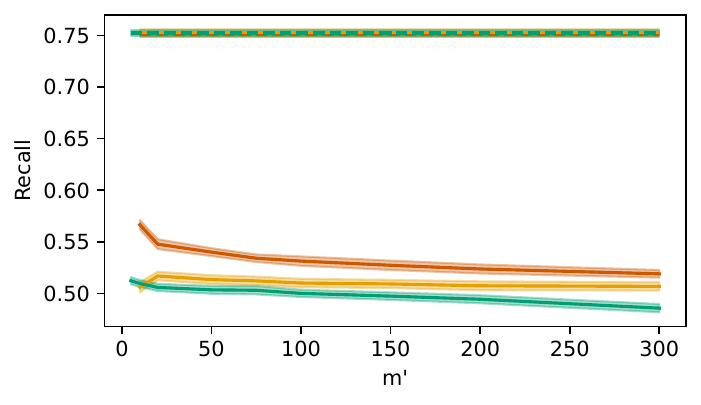}
\includegraphics[height=3.4cm]{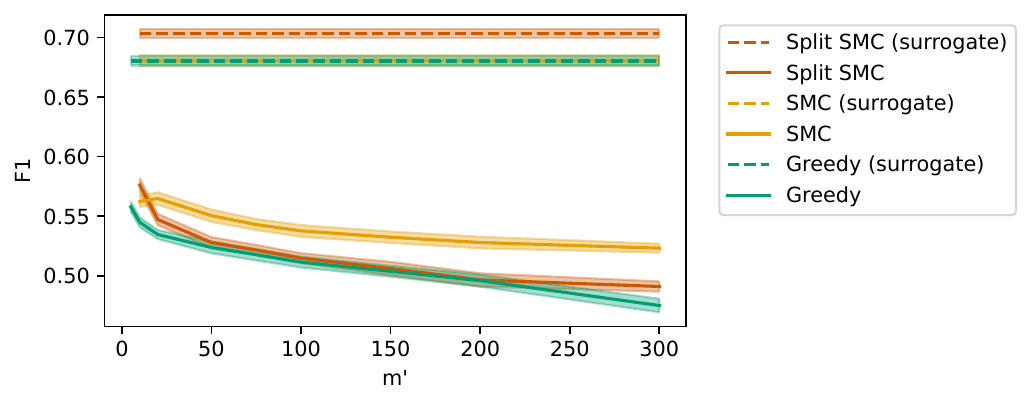}

    }
    \caption{Performance metrics against surrogate proposal size $m'$ on the TweetNERD dataset. The mean across replications is plotted, with a shaded region indicating a 95\% confidence interval. Dashed horizontal lines indicate results using the surrogate model alone.}
    \label{fig:tweet-evals}
\end{figure}

\FloatBarrier
\section{PROOFS}\label{sec:proofs}

\subsection{Posterior Factorisation}\label{sec:splittingproof}

We restate and prove a more formal version of Proposition \ref{prop:splitting}, motivating the splitting approach used in our split SMC algorithm.

By Ewens’s sampling formula \citep{ewens1972sampling}, the posterior probability of a dataset $X=x_{1:t}$ having the cluster labels $z_{1:t}$ can be written as
\begin{align}
  p(z_{1:t}\mid x_{1:t})
  \;\propto\;
  \underbrace{\frac{\alpha^{K-1}}{\prod_{i=0}^{t-1}(\alpha+1+i)}
  \Bigl(\prod_{k=1}^K \Gamma(|c_k|)\Bigr)}_{\rm prior} \ 
  \underbrace{\Bigl(\prod_{k=1}^K p(c_k)\Bigr)}_{\rm likelihood}.\label{eq:dpmm-post}
\end{align}

Here, $c_k$ refers to the collection of datapoints given cluster label $k$, $|c_k|$ is the number of datapoints in cluster $k$, and $p(c_k)$ refers to the likelihood of observing the given unordered collection in a sample of that size. It can be obtained by integrating out some unknown parameter $\theta$, conditioned on which the datapoints within a cluster are independent of each other. 

Suppose we partition the dataset $X$ into disjoint subsets $E^1$ and $E^2$. For a given clustering, if none of the clusters $c_k$ contains elements from both $E^1$ and $E^2$, and without loss of generality we relabel the clusters so the first $J$ contain elements of $E^1$ and the last $K-J$ contain elements of $E^2$, then the posterior density factorises:
\[
  p(z_{1:t}\mid x_{1:t})
  \;\propto\;
  \Bigl(\alpha^{J-1}\!\!\prod_{k=1}^J  \Gamma(|c_k|)\,p(c_k)\Bigr)
  \;\times\;
  \Bigl(\alpha^{K-J-1}\!\!\prod_{k=J+1}^K  \Gamma(|c_k|)\,p(c_k)\Bigr),
\]
where the first product runs over the clusters contained in $E^1$ and the second over those in $E^2$. Conditioning on the event that no cross‐subset clusters occur thus yields
\[
  p(z_{1:t}\mid x_{1:t},\,\text{no‐overlap})
  \;=\;
  p(z_{E^1}\mid x_{E^1} ,\,\text{no‐overlap})\;\times\;p(z_{E^2}\mid x_{E^2} ,\,\text{no‐overlap}).
\]

Since the true posterior factorises for any clustering consistent with a given partition, the following result applies:

\begin{proposition}
Suppose we have a partition of the dataset $X$ into two elements $E^1$ and $E^2$. \\ We index possible clusterings of $E^1$ by $i=1,\ldots,I$ and clusterings of $E^2$ by $j=1,\ldots,J$, writing $p_{ij}$ for the true posterior probability of clustering $(i,j)$. Let $p$ factorise so that $p_{ij}=p^1_ip^2_j$ $\;\forall (i,j)$. Suppose we have a distribution $\hat{p}$ with $\text{support}(\hat{p}) \subseteq \{1,\ldots,I\}\times\{1,\ldots,J\} \subseteq \text{support}(p)$ and marginal distributions $\{b_i\}$ and $\{c_j\}$ over clusterings of $E^1$ and $E^2$. Then, the distribution $q$ with marginals $\{b_i\}$ and $\{c_j\}$ that minimizes $KL(q||p)$ is the product of these marginals, $q_{ij}=c_ib_j$ $\forall i,j$. 
\end{proposition}
\begin{proof}
\begin{align*}
    \text{KL}(q \| p) &= \sum_{i,j} q_{ij} \log{\frac{q_{ij}}{p_{ij}}}
    = \sum_{i,j} \left(q_{ij} \log{q_{ij}} - q_{ij}\log p_{ij}\right)\\
    &= \sum_{i,j} \left(q_{ij} \log{q_{ij}} - q_{ij}\log p^1_i - q_{ij}\log{p^2_j} \right)\\
    &= \sum_{i,j} q_{ij} \log{q_{ij}} - \sum_i \log(p^1_i) \sum_j q_{ij} - \sum_j \log(p^2_j) \sum_i q_{ij}\\
    &= \underbrace{\sum_{i,j} q_{ij} \log{q_{ij}}}_{-H(q)} - \underbrace{\sum_i b_i \log p^1_i - \sum_j c_j \log p^2_j}_{\text{constant}}\\
\end{align*}
The first term is the negative Shannon entropy $H$ of $q$, whilst the second depends only on the marginals of $p$ and $q$, which are fixed. Since $q$ is a coupling of $b$ and $c$, we have that $H(q) \leq H(b) + H(c)$, with equality achieved when $q=b\otimes c$ \citep{cover1999informationtheory}. Hence, taking $q$ to be the product of the marginals of $\hat{p}$ minimises KL divergence. Note that we take the reverse KL here, since for a particle approximation we typically have $\text{support}(\hat{p}) \subsetneq \text{support}(p)$ and hence the forward KL, $\text{KL}(p||\hat{p})$, is infinite.
\end{proof}
By induction, this result also holds for partitions into more than two elements such that $\hat{p}$ gives a probability of zero to any clustering where two datapoints on different partition elements are in a cluster together.

\subsection{Greedy Resampling}\label{sec:resamplingproofs}

In this section, we provide some theoretical motivation for using greedy resampling in SMC, and for discarding particles that introduce overlap in the split representation. 

\begin{proposition}
Let $x$ and $y$ be discrete probability distributions with $\mathcal{Y} := \text{support}(y)\subseteq \text{support}(x)$. Then $\text{KL}(y||x) \geq -\log\sum_{i\in\mathcal{Y}} x_i$.
\end{proposition}
\begin{proof}
Proved as part of Theorem 2.6.3 in \cite{cover1999informationtheory}:
\begin{align*}
    - \text{KL}(y||x) &= \sum_{i\in\mathcal{Y}} y_i \log{\frac{x_i}{y_i}}\\
    &\leq \log \sum_{i\in\mathcal{Y}} y_i \frac{x_i}{y_i} \quad (\text{Jensen's inequality})\\
    &= \log \sum_{i\in\mathcal{Y}} x_i
\end{align*}
\end{proof}
This lower bound is achieved if  $x_i= Z y_i$ for all $i \in \mathcal{Y}$, for the constant $Z = \sum_{i\in\mathcal{Y}} x_i$. In this case,
\begin{align*}
    \text{KL}(y||x) &= \sum_{i\in\mathcal{Y}} y_i \log{\frac{y_i}{x_i}}
     = \sum_{i\in\mathcal{Y}} y_i \log\frac{y_i}{Zy_i}
     = -\sum_{i\in\mathcal{Y}} y_i \log Z
     = -\log Z \, .
\end{align*} 

If we choose the support of $y$ subject to the condition $|\mathcal{Y}|=m$, this means that $KL(y||x)$ can be minimised by choosing the $m$ support points of $x$ with the highest weight, since this maximises the value of $\log Z$. Hence, we have that choosing the support of $y$ in this way and setting its weights to be proportional to the weights in $x$ is optimal in terms of KL divergence to $x$. This forms a ``greedy resample" of size $m$ from $x$.

Hence, for a single resample of $m$ particles within a particle filter:
\begin{itemize}
\item The resampled distribution with the lowest KL divergence to the original is the one obtained by a greedy resample.
\item If the weights of the particle approximation are proportional to the true posterior mass function, the optimal resample of size $m$ from that approximation is a greedy resample, in terms of KL divergence to the true posterior distribution.
\end{itemize}
If greedy resampling is used at each iteration of the particle filter, then its weights will always be proportional to the true posterior mass function. This does not guarantee that applying greedy resampling \textbf{sequentially} will result in an optimal approximation to the true posterior. However, empirical results on the data considered in our experiments showed that particle filters using greedy resampling were able to find clusterings with higher log-posterior mass than those using stochastic resampling.

In addition, in the context of our split SMC algorithm, it follows that it is sometimes optimal to discard particles that introduce overlap between subproblems rather than to perform a merge by taking a sample from their joint distribution (discussed in Sections \ref{sec:merging} and \ref{sec:implementation}). If we assume that the particle filter's weights are proportional to the true posterior, then any resample has KL divergence greater than or equal to that of a greedy resample of the top $m$ clusterings in their joint distribution. If all of the assignments from a subproblem $s$ would be discarded in this greedy resample, then the support of the greedily resampled set is a subset of that of the particle set where we retain the original subproblem factorisation by discarding only the particles that assign the new observation to clusters in $s$. Hence the latter particle set has lower KL divergence to the true posterior, and it is optimal not to perform the merge (even if the resampler used in the merge is not greedy). 

While checking this condition does not require the construction of the explicit joint distribution in full, it is nonlinear in the number of subproblems under consideration. In practice, faster approximate checks are sufficient for improved accuracy -- the rule of thumb we suggest in Section \ref{sec:implementation} discards particles assigning an observations to subproblem $s$ if their combined weight is less than $1/m$, since in this case we expect them to be discarded in the merge.

\section{SPLIT SMC ALGORITHM DETAILS}\label{sec:smc-details}

\begin{figure}
\centering
\includegraphics[width=0.9\textwidth]{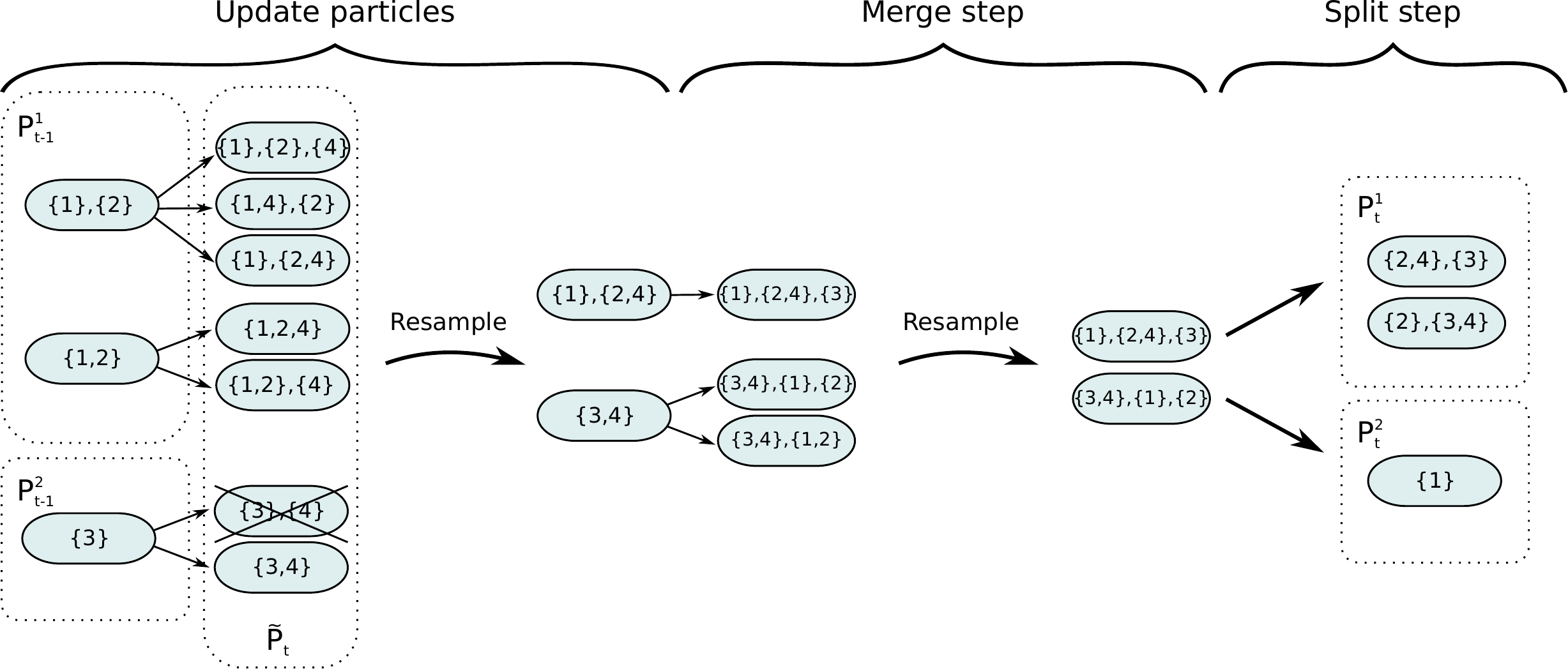}
\caption{Illustration of a state update in split SMC.}\label{fig:update}
\end{figure}

Here, we give a detailed description of the full update step in split SMC. An illustration is given in Figure \ref{fig:update}, where we suppose that $m=2$ and we start with a particle set containing clusterings of previous observations $\{1,2,3\}$, split into the subproblems $\{1,2\}$ and $\{3\}$. We refer to this illustration to give a concrete example for each part of the update.

\begin{enumerate}[label=\textbf{\arabic*.}]
    \item \textbf{Update particles}
\begin{itemize}
    \item Construct the putative particle set $\tilde{\mathcal{P}}_t$, containing a putative particle $\tilde{p}_t(s,i,c)$ for each $1 \leq s \leq S$, $1 \leq i \leq |\mathcal{P}^{(s)}_{t-1}|$, and $c \in p_{t-1}(s,i) \cup \{\varnothing\}$
    \item Compute putative weights $\tilde{w}_t(s,i,c) = w_{t-1}^{(i)} \;
  p\bigl(x_t\text{ belongs to } c \mid  p_{t-1}(s,i),x_{1:(t-1)}\bigr)$
    \item Discard duplicate singleton assignments: for $c=\varnothing$ keep only the $\tilde{p}_t(s,i,c)$ for the subproblem containing the highest value of $p\bigl(x_t\text{ belongs to } c \mid p_{t-1}(s,i), x_{1:(t-1)}\bigr)$
    \item Resample the $m$ particles with the highest value of $\tilde{w}_t(s,i,c)$ from $\tilde{\mathcal{P}}_t$.
\end{itemize}

\ExampleBox{
Having observed new datapoint $4$, a putative particle is constructed for each possible cluster assignment. Since the particle set represents a Cartesian product, the putative particle $\{\{3\},\{4\}\}$ on the second subproblem implicitly represents the particles $\{\{3\},\{4\},\{1\},\{2\}\}$ and $\{\{3\},\{4\},\{1,2\}\}$, which are also contained in the expanded particle set from $\mathcal P^1$. The choice of which set of singletons is retained is arbitrary, so we choose the set on the subproblem with the most likely cluster assignment for $x_t$ to reduce the likelihood of needing to perform a merge later. $m=2$ putative particles are retained after resampling.}

\item \textbf{Merge}

If the resampled particles all have the same value of $s$, then we set $\mathcal P^s_t=\tilde{\mathcal{P}}_t$. Otherwise:
\begin{itemize}
    \item Discard particles with a value of $s$ such that $\sum_{i,c}\tilde{w}_t(s,i,c) \leq 1/m$
    \item If only one value of $s$ remains, pass to split step
    \item If at most two values of $s$ have $|P^s_{t-1}|>1$, compute the full joint distribution over the $\tilde{P}^s_t$. This is done by combining each $\tilde{p}_t{(s,i,c)}$ with every possible particle configuration $p'$ drawn from the Cartesian product $\bigotimes_{j\in\mathcal{S}\setminus\{s\}}\mathcal{P}^j_{t-1}$. Then we have for each combination and weight
\begin{align*}
    p_{\rm joint} &= \tilde{p}_t{(s,i,c)} \;\cup\; p', \\
    w_{\rm joint} &= \tilde{w}_t{(s,i,c)} \;\times\; \prod_{j\in\mathcal{S}\setminus\{s\}} w_{t-1}^{(j,\cdot)}\,.
\end{align*}
    Select the $m$ top-weighted clusterings in the above joint distribution.
    \item Else, perform a multinomial merge: (i) sample $m$ assignments of $x_t$ (with replacement) from the marginal distribution over assignments, then (ii) for each sampled assignment, independently draw one particle from each of the other involved subproblems, and finally (iii) merge any duplicate particles.
    \item Finally, delete the $\mathcal P^s_{t-1}$ from the particle representation, replacing them with a new subproblem consisting of the selected particles.
\end{itemize}

\ExampleBox{
The resample step selected one particle from each subproblem, so a merge occurs. In our implementation, multinomial resampling is only used in merges of more than two subproblems, where at least 3 of the subproblems have more than one particle, to ensure that the computational cost of the merge step is at most quadratic in $m$. In practice, this occurs extremely rarely.

Here, only two subproblems are involved, so we form the joint distribution over clusterings of $\{1,2,3,4\}$ by expanding each particle as a Cartesian product with the old generation of particles on the other subproblem. Then, a resample is carried out to bring the number of particles on the merged subproblem down to $m=2$.
}

\item \textbf{Split}
\begin{itemize}
    \item For the subproblem $\mathcal P^s_t$ to which $x_t$ was assigned, construct a graph whose vertices are observations and whose edges connect any two points that co‐occur in a cluster on at least one particle
    \item Decompose this graph into its connected components, $E_k$
    \item If there is more than one, delete the subproblem and create a new subproblem for each component as
\begin{align*}
  \mathcal P_k
  \;&=\;
  \{\,p\cap C_k \mid p\in\mathcal P^s_t\}, \\
  w^{(k,i)}
  &=\sum_{j=1}^{|\mathcal P_t|}w_t^{(j)}\,\mathds{1}[\,p_t^{(k,i)}\subseteq p_t^{(j)}\,]\,,
\end{align*}
    where $C_k$ is the set of clusters involving points in $E_k$,
    \begin{align*}
  C_k
  \;=\;
  \bigl\{\,c\in\bigcup_{p\in\mathcal P_t}p : c\subseteq E_k\bigr\},
\end{align*}
\end{itemize}

\ExampleBox{
Observation 1 does not appear in a cluster with observations 2, 3, or 4 on any particle, so we perform a split assigning it to a separate subproblem.
}
\end{enumerate}

\subsection{Surrogate Proposal}
When using a surrogate model as a proposal distribution, $m$ particles with non-empty cluster assignments for the new observation are selected by greedy resampling from the putative particle set, with weights determined by using the surrogate likelihood in the weight update \eqref{eq:marginaldist}. The putative particles where the new observation is assigned to a singleton cluster are also retained. After this proposal step, all particles are re-weighted according to the model likelihood and a further resample down to a total of $m$ particles is carried out. We handle singleton assignments in this way because the value of the Dirichlet process concentration $\alpha$ is typically tuned according to the model likelihood and hence may be poorly calibrated to the surrogate likelihood. Keeping the singleton assignment particles until the model resampling step prevents them from being discarded prematurely by the surrogate model, ensuring that assignments to non-empty clusters are only weighted highly in the final particle set if they are more likely than the singleton assignment.

If $m$ is small or there is considerable mismatch between the surrogate model and the target model, it can be helpful to propose more than $m$ particles from the surrogate distribution. The size of the proposal step can be chosen according to model evaluation budget, since the number of model evaluations required is at most $m'+1$, where $m'$ is the number of non-singleton particles selected in the proposal step. We investigate the impact of varying $m'$ in Section \ref{sec:surrogate-ablation}. We set $m'=m$ in the rest of our experiments, so that the number of model evaluations scales with the size of the particle set. For the greedy algorithm, $m'$ is set according to evaluation budget since in this case $m$ is fixed at 1. We use $m'=100$ for the text data, and do not use a surrogate model in greedy clustering of the circles data.

\section{EXPERIMENTAL DETAILS}\label{sec:experiment-details}
The source code for our experiments is available on GitHub at \href{https://github.com/microsoft/smc-clustering}{https://github.com/microsoft/smc-clustering}. Algorithms were implemented in Python 3 with JAX \citep{jax2018github}, SciPy \citep{2020SciPy-NMeth} and NumPy \citep{harris2020numpy}. The neural model for the circles data was implemented in Flax \citep{flax2020github} and the neural language model was implemented in PyTorch \citep{2019pytorch}. In all of our algorithm implementations, likelihood evaluations are cached to avoid additional computational cost from re-computing likelihoods. The two REBEL experiments were run on a Linux virtual machine with Ubuntu 24.04.2 LTS on GPU, an A100 with 80Gb of DRAM. All other experiments were run on a cluster with the 22.04.5 LTS ``Jammy Jellyfish" 64 bit operating system on CPU (Intel(R) Xeon(R) CPU E5-2699 v3 @ 2.30GHz) with a memory limit of 10Gb allocated per replication.

\subsection{Offline Algorithms}\label{sec:algorithm-details}

Offline methods were initialised with the fully disconnected clustering, where every observation is in a singleton cluster. For MCMC samplers, other initialisations are possible, such as the fully-connected clustering that places all observations in a single cluster, and can be helpful for mixing in some applications. The disconnected initialisation is more practical for the clustering problems we consider for two reasons: they have a large underlying number of clusters, and neural likelihood models have an evaluation cost that increases with cluster size.

\paragraph{MCMC} We use a Gibbs sampler \citep{neal2000dpmmmcmc} where at each iteration a Gibbs update is carried out on each datapoint's assignment, in a random order. This involves sampling a new cluster assignment from the marginal distribution defined in \eqref{eq:marginaldist}. For the problems with neural likelihoods, we also considered a Metropolis-within-Gibbs variant. In this sampler, the new assignment is proposed using the surrogate likelihood and is then accepted or rejected according to a Metropolis-Hastings step. This process substantially reduces model evaluations relative to ordinary Gibbs sampling, but can have slower mixing in terms of the number of iterations required for convergence. We compared the two approaches in the circles experiment (see Section \ref{sec:runtime-metrics}), and found that Metropolis-within-Gibbs converged in a much shorter time and obtained slightly superior results to ordinary Gibbs, which did not appear to converge fully in the given runtime.

\paragraph{Agglomerative Clustering} In Table \ref{tab:metrics}, we presented results for a full-batch agglomerative algorithm that computes model likelihoods for each possible merged cluster at each iteration. This method has an extremely high runtime due to the number of model evaluations required, and it is possible to formulate a faster, approximate agglomerative algorithm by instead sampling batches of cluster merges to evaluate at each iteration. We present results for a range of batch sizes in Section \ref{sec:runtime-metrics}. They achieved comparable but typically inferior results at lower runtimes than full-batch agglomerative clustering, but were slower than the online algorithms. 

\subsection{Experiments}\label{sec:data-details}

\paragraph{Circles} 15 cluster centres were drawn from a 2D uniform distribution on $[-5,5]\times[-5,5]$, with datapoints drawn uniformly from the circumference of a circle of radius 0.6. The number of points in each cluster was drawn from a uniform categorical distribution on $\{10,\ldots,30\}$, resulting in a dataset of size 306.

The variational diffusion model \citep{kingma2021variationaldiffusion}  was trained on a simulated dataset of 100~000 clusters for 150 epochs using the Adam optimiser. The score function model was a transformer with self attention applied across the sequence of observations within each cluster in order to produce a model for collections of observations \citep{Vaswani-etal-2017-attention}. Since we do not use positional encoding, the resulting likelihood model is invariant to permutation of its inputs, fulfilling the criteria we assume for cluster likelihoods in Section \ref{sec:Bayesian-nonparametrics}. The architecture used was a residual neural network with 6 residual blocks, each consisting of a self-attention layer, a dense layer of dimension 8, and a dense layer of dimension 2. GELU (Gaussian Error Linear Unit) activation and RMSNorm  \citep{zhang-sennrich-2019-rmsnorm} were used, with dropout applied during training. A linear noise schedule was used. Cluster likelihoods were estimated by numerical integration of the probability flow ODE. The full implementation of this model can be found in the source code for this paper.

The prior parameters in the surrogate model were tuned by numerically optimising its model evidence on a simulated dataset of 1000 clusters. The Dirichlet process concentration was set to $\alpha=1$ for all of the clustering algorithms. 

Results are reported over 30 replications of the clustering algorithms. SMC samplers were run for a range of particle set sizes with runtimes up to 1000s. MCMC samplers were given a maximum of 10~000s, stopping early if no change was made to the estimated MAP clustering for 500 sweeps. The agglomerative algorithms were run until all potential merges were exhausted, or until 100 iterations passed with no changes to the clustering. Note that the model log-likelihood evaluations are noisy due to use of the Skilling-Hutchinson trace estimator when computing cluster likelihoods. This leads to variability in the results of otherwise deterministic algorithms. Full results for this dataset are plotted in Figure \ref{fig:circles-runtime}.

\paragraph{Gaussian mixture} Cluster sizes for 100 clusters and 700 datapoints were generated from a Dirichlet process with concentration $\alpha=20$ using the stick-breaking construction. Some of the resulting clusters were empty, so the number of clusters present in the generated dataset was 80. The cluster data were generated from Gaussian distributions with independent components. In order to create a problem with multiple close groups of clusters, means were generated in a hierarchical fashion. First, cluster variances were drawn from an inverse Gamma distribution,
\begin{align*}
\sigma_i^{-2} &\sim\Gamma(\alpha,\beta)\,.
\end{align*}
Then 16 group centres were drawn from a Normal distribution,
\begin{align*}
    \mu_g \sim N(\mu,\lambda^{-1})\,.
\end{align*}
Each cluster mean $\mu_i$ was set by selecting one of the $\mu_g$ uniformly at random and applying an additional Normal perturbation,
\begin{align*}
    \mu_i \sim \mu_g + N\left(0,\frac{\sigma_i^2}{125\lambda}\right)\,.
\end{align*}
The parameters used were $\alpha=2$, $\beta=0.5$, $\mu=0$, and $\lambda=0.0002$. In clustering, cluster locations and variances were modelled with a Normal-inverse-Gamma distribution with these parameters.

The MCMC sampler was given a maximum of 10~000s, stopping early if no change was made to the estimated MAP clustering for 500 sweeps. The agglomerative algorithms were run until all potential merges were exhausted, or until 100 iterations passed with no changes to the clustering. These results are plotted in Figure \ref{fig:gmm-runtime}.

\subsubsection{Text Data}\label{sec:kb-details}

\paragraph{Neural model} The cluster likelihood model scores a set of fragments by canonically serializing them into JSON and summing grammar-constrained token log-probabilities under a decoder-only Transformer ($\approx$8.8M parameters). We use a hybrid tokenizer with a structural symbol dictionary and BPE pieces applied inside quoted string literals \citep{sennrich-etal-2016-bpe}. The model is GPT-style \citep{Vaswani-etal-2017-attention}: d\_model=256, 10 layers, 8 heads (32-dim), SwiGLU \citep{shazeer2020gluvariantsimprovetransformer} feed-forward (d\_ff=768), RMSNorm pre-norm \citep{zhang-sennrich-2019-rmsnorm}, dropout 0.05, RoPE positional encoding \citep{su-etal-2024-rope}, and tied embeddings. A finite-state grammar automaton builds an allowed-token mask at each step, so the sequence log-likelihood is the sum of log-probabilities over valid transitions only \citep{koo2024automatabased}. Cluster (set) likelihoods are the products of these constrained conditional probabilities over the canonical serialization. Training used AdamW (lr 3e-4, weight decay 0, 1500 warmup steps, cosine decay with 0.1 floor) for 10 epochs on 1.6M training examples from the REBEL train split.

\paragraph{Model calibration} While $\alpha$ can be tuned by empirical Bayes methods \citep{escobarwest1995bayesmix,aucliffe2006empiricalbayes}, this process assumes that the model likelihood is both normalised and well-calibrated, since $\alpha$ is adjusted to make clusterings with high likelihood have high prior probability. We found that for neural models, and neural language models in particular, the resulting values for $\alpha$ did not necessarily give good accuracy. Since rescaling $\alpha$ is equivalent to multiplying each cluster likelihood by a constant in the clustering posterior distribution, tuning $\alpha$ can be seen as equivalent to tuning the normalising constant of the likelihood. We used a labelled validation clustering dataset from REBEL to find a likelihood rescaling that maximised the log-likelihood of the ground truth clustering. The chosen rescaling  was equivalent to setting $\alpha=500$.

\paragraph{N-gram surrogate model} The surrogate model used for clustering entity fragments was a character-level bigram model on the \textit{name} attribute. A Dirichlet prior was placed on the probabilities $\theta_{i|h} = \PP(\text{next character is } i | \text{previous character is } h)$ for each character history $h$, as in \cite{chien2015bayesngram}. The prior parameters $\alpha_{i|h}$ of these Dirichlet distributions were chosen to be a rescaling of the posterior parameter of a fitted Dirichlet-Categorical distribution. Observed n-gram counts were taken from a training subset of the REBEL data, and a uniform Dirichlet prior was used. This is equivalent to setting $\alpha_{i|h} = 1 + \text{count}((h,i))$, the observed count of n-gram $(h,i)$ with plus-one smoothing. A constant multiplicative rescaling factor $c\in(0,1]$ was used on each $\alpha_{i|h}$ to ensure that the prior distribution in the DPMM was not too concentrated on the fitted n-gram frequencies, allowing for more variation in these frequencies between clusters. This was tuned by numerically optimising the model evidence of the true clustering on the labelled training data.

\paragraph{REBEL} Our experiments used two datasets (REBEL-50 and REBEL-200) derived from the REBEL dataset \citep{huguet-cabot-navigli-2021-rebel}, which is released under a CC BY-NC-SA 4.0 license. This dataset potentially contains personal/sensitive information about people, as its contents are taken directly from text publicly available on Wikipedia. We construct our entity clustering dataset by starting from the full set of REBEL “fragments” (extractions from short text spans) and downsampling to a fixed target size while preserving realistic ambiguity. A uniform fragment sample would yield mostly singleton or trivially separable clusters, and picking random entities would still tend to give clusters whose members differ too much to be challenging. Instead, we define “confusing entities” as distinct entities that share at least one surface form (a matching mention string) with a seed entity, making them plausible collision candidates. We then iteratively build the dataset by selecting 2-5 confusing entities at a time and collecting each entity's fragments (capped at 50) until the total number of unique entities reaches the target. For splitting, each group (a seed entity plus its confusing entities and all of their fragments) is assigned wholesale to a single train, dev, or test split.

\paragraph{TweetNERD} We construct a clustering dataset from the TweetNERD-Academic dataset of \cite{mishra2022tweetnerd}, a benchmark dataset in named entity recognition and disambiguation. It is released under a CC BY 4.0 license. TweetNERD-Academic consists of free-text entity mentions extracted from tweets, annotated with corresponding entities in the Wikidata knowledge base \citep{wikidata}. We collate the list of name mentions for each entity, and our clustering dataset is made up of the unique name variations for the 200 largest entities, with resulting cluster sizes from 4 to 148, and 1931 observations in total. In clustering, these are treated as entity fragments, each with a single observation of the \textit{name} attribute.

\FloatBarrier

\end{document}